\documentclass[preprint,12pt]{elsarticle}

\usepackage[utf8]{inputenc}
\usepackage[T1]{fontenc}
\usepackage{lmodern}
\usepackage{hyperref}
\usepackage{url}
\usepackage{booktabs}
\usepackage{amsfonts}
\usepackage{nicefrac}
\usepackage{microtype}
\usepackage{xcolor}
\usepackage{enumitem}
\usepackage{listings}

\usepackage{pgfplots}
\pgfplotsset{compat=1.18}
\usepgfplotslibrary{groupplots}
\usepackage{graphicx}
\usepackage{subfigure}
\usepackage{amsmath}
\usepackage{amssymb}
\usepackage{mathtools}
\usepackage{bbm}
\usepackage{amsthm}
\usepackage[capitalize,noabbrev]{cleveref}
\usepackage{algorithm}
\usepackage{algorithmic}
\usepackage{lineno}

\allowdisplaybreaks
\theoremstyle{plain}
\newtheorem{theorem}{Theorem}[section]
\newtheorem{proposition}[theorem]{Proposition}
\newtheorem{lemma}[theorem]{Lemma}
\newtheorem{corollary}[theorem]{Corollary}
\theoremstyle{definition}

\theoremstyle{remark}

\journal{Neural Networks}

\begin{document}

\begin{frontmatter}

\title{ANCORA: Learning to Question via Manifold-Anchored Self-Play for Verifiable Reasoning}

\author[whu]{Chengcao Yang\corref{cor1}}
\ead{adammyhos@gmail.com}
\cortext[cor1]{Corresponding author.}
\affiliation[whu]{organization={Wuhan University},
            city={Wuhan},
            country={China}}

\begin{abstract}
We propose a paradigm shift toward \textbf{open-ended curriculum self-play}: rather than learning to answer on a fixed prompt set, a unified policy learns to question---generating verifiable problems, solving them, and turning verifier feedback into self-improvement. We introduce \textbf{ANCORA}, in which the policy alternates between a Proposer that synthesizes novel specifications and a Solver that produces verified solutions, anchored by three mechanisms: a \textbf{two-level group-relative update} coupling Proposer advantages across specifications with Solver advantages across solution attempts; \textbf{iterative self-distilled SFT} projecting the base model onto its valid-output manifold before RL; and a \textbf{UCB-guided Curriculum DAG} whose policy-induced problem set can provably expand under self-composition. Without these stabilizers, sparse verifier feedback drives Proposer collapse even under MLRL-aligned rewards; with them, ANCORA bootstraps a verifiable curriculum from curated seeds. Instantiated in Verus, ANCORA lifts Dafny2Verus pass@1 from a 26.6\% SFT baseline to 81.5\% in test-time training (TTT, 0-shot), outperforming PSV self-play (1-shot) by 15.8 points; in a transfer setting, training from Dafny2Verus seeds yields 36.2\% and 17.2\% pass@1 on MBPP and HumanEval.
\end{abstract}


\end{frontmatter}

\section{Introduction}
Large language models have achieved broad competence via pre-training on web-scale corpora, supervised fine-tuning (SFT), and reinforcement learning from verifiable feedback \cite{deepseekai2025deepseekr1incentivizingreasoningcapability}. However, this paradigm faces converging bottlenecks: public human-generated text may become a scaling bottleneck if current data-demand trends continue \cite{villalobos2024will}, human annotation does not scale, and imitation-based training caps performance at human level \cite{singhHumanDataScaling2024}. Tool-augmented reasoning frameworks such as ReAct \cite{ReACT} broaden how models interact with external environments, but they still typically assume a given task distribution rather than building a verifier-backed curriculum of newly generated problems.

We propose \textbf{ANCORA}\footnote{Code is available at \url{https://github.com/MythosAd/ANCORA}.}, an anchored-curriculum framework that shifts from \emph{learning to answer} to \emph{learning to question}. A unified policy $\pi_\theta$ alternates between a Proposer that synthesizes novel problem specifications and a Solver that produces verified solutions, coupling both roles through a shared internal model of the problem--solution space. The framework has three components: (i)~a \textbf{two-level group-relative update} coupling Proposer advantages across candidate specifications with Solver advantages across solution attempts; (ii)~\textbf{iterative self-distilled SFT} that lifts the base model onto its valid-output manifold before full RL; and (iii)~a \textbf{UCB-guided Curriculum DAG} that grows only through strictly filtered, novel, Solver-verified specifications, allowing the prompt pool to expand without poisoning the valid manifold.

\begin{figure}[t]
\centering
\includegraphics[width=\linewidth]{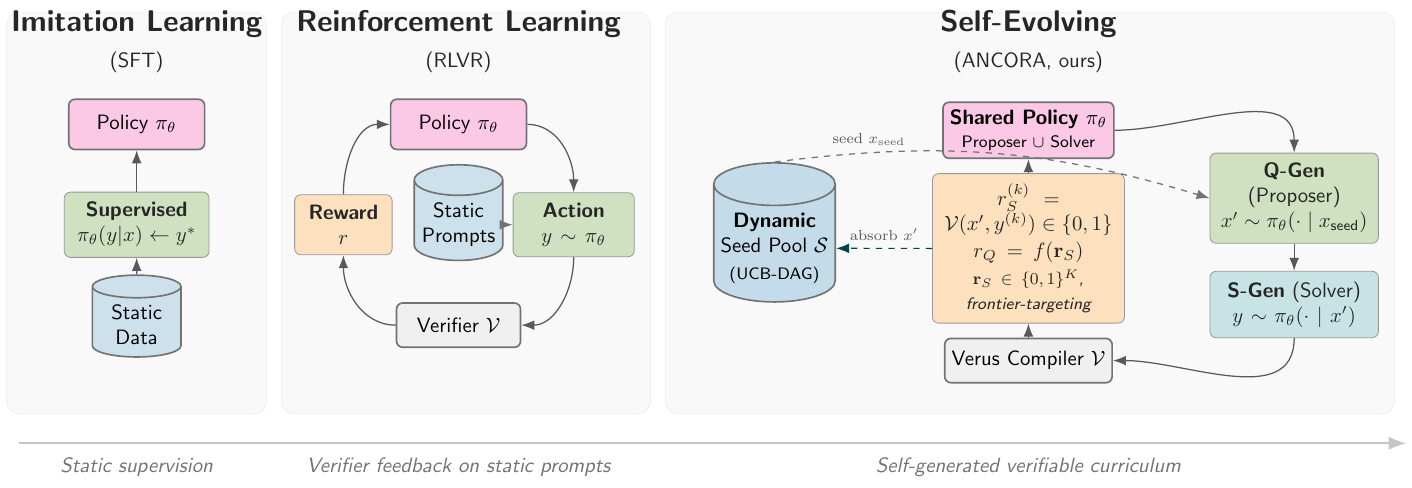}
\caption{Learning paradigms. \textbf{Left:} SFT is bounded by static-data supervision. \textbf{Middle:} RLVR uses environment feedback but a fixed prompt pool. \textbf{Right:} ANCORA. The Proposer (Q-Gen) generates a novel spec $x'$ from seed $x_{\text{seed}}$ via 1-shot ICL; the Solver (S-Gen) attempts implementation $y$. A coupled MLRL update shapes both personas of $\pi_\theta$ via Solver outcomes $\mathbf{r}_S\in\{0,1\}^K$ (\S\ref{sec:proposer_obj}--\S\ref{sec:solver_obj}). Strictly filtered, novel, Solver-verified specs re-enter the Dynamic Seed Pool, yielding a self-growing curriculum beyond the initial seed set.}
\label{fig:gsv_diagram}
\end{figure}

The joint update remains close to GRPO \cite{shao2024deepseekmathpushinglimitsmathematical}, but its mean normalization is motivated by \textbf{Maximum-Likelihood RL} (MLRL) \cite{tajwar2026maximumlikelihoodreinforcementlearning}: for binary Solver rewards, it matches the likelihood-aligned weighted-$\textsc{pass}@k$ estimator up to a score-function baseline. This matters because recent pass@$k$ analyses suggest that conventional RLVR can mainly sharpen small-$k$ sampling efficiency without expanding, and sometimes while narrowing, the large-$k$ reasoning boundary already present in the base model \cite{yue2025doesreinforcementlearningreally}. For the Proposer, the goal is complementary: generate specifications near the Solver's epistemic frontier, where outcome entropy is high enough to create useful Solver learning signal.

A central failure diagnosis is \textbf{manifold collapse}: when valid outputs are sparse relative to the full generation space, sparse verifier feedback can push the Proposer off the valid manifold regardless of whether the Proposer reward is Bernoulli variance $4p(1{-}p)$, a low-pass exponential surrogate $e^{-Kp}{-}e^{-K}$, or a Band-1-of-$K$ indicator $\mathbbm{1}[m{=}1]$ (\ref{app:manifold_collapse}). We use ``manifold'' operationally, for the set of outputs accepted by the relevant hard check, and distinguish this validity collapse from diversity contraction (Section~\ref{sec:failure_modes}). Iterative SFT and the strictly gated, seed-anchored curriculum are therefore not auxiliary engineering details; they are the stabilizers that let frontier-targeted Proposer rewards operate without curriculum poisoning. Under these stabilizers, the Band-1-of-$K$ reward is our main-result default; we ablate the other two shapes in \ref{sec:reward_ablation}.

We validate ANCORA on \textbf{Verus} formal verification \cite{lattuada2023verus}, where a deterministic compiler provides scalable ground-truth rewards. Compared to prior self-play and self-improvement methods \cite{zhao2025absolutezeroreinforcedselfplay, wilf2025proposesolveverifyselfplay, chen2025selfquestioninglanguagemodels}, ANCORA is distinguished by its combination of coupled two-level optimization, manifold-first initialization, and strictly gated UCB-DAG exploration.

Our contributions are:
\begin{enumerate}[leftmargin=1.6em,itemsep=1pt,topsep=2pt]
\item We introduce ANCORA, the first framework, to our knowledge, that integrates a unified-policy Proposer/Solver with the three mechanisms above for verifiable reasoning.
\item We diagnose the \emph{manifold collapse} problem for sparse-verifier RL and analyse why manifold projection plus gated anchored exploration supports stable Proposer learning.
\item We show that a Qwen2.5-Coder-3B model trained with ANCORA outperforms the PSV self-play baseline on Dafny2Verus and MBPP, including under PSV's own prompt format, and we report data-provenance and overlap audits.
\end{enumerate}
The study is limited to one verifier, one 3B model, and single training runs (Section~\ref{sec:limitations}).

\section{Related Work}

\textbf{Reasoning with Static Prompts.}
STaR \cite{zelikman2022starbootstrappingreasoningreasoning}, OpenAI o1 \cite{openai2024openaio1card}, and DeepSeek-R1 \cite{deepseekai2025deepseekr1incentivizingreasoningcapability} train reasoning via RL on a \emph{fixed} set of human-curated prompts.
While effective at distilling search compute into the policy, they do not expand the problem space itself, and performance is bounded by the cardinality of the initial prompt set.

\textbf{Concurrent Self-Play Methods.}
AZR \cite{zhao2025absolutezeroreinforcedselfplay} uses a single model for code-based reasoning with a proposer reward of $1 - \bar{r}_{\text{solve}}$ and Task-Relative REINFORCE++~\cite{hu2025reinforceplusplus}.
SPELL \cite{spell2025selfplay} adds three roles (questioner, responder, verifier) for document QA via summed GRPO losses.
SPICE \cite{liu2025spice} applies dual-role Dr.\ GRPO~\cite{liu2025drgrpo} to corpus-grounded math reasoning.
SSR \cite{wei2025ssr} trains a shared bug-injector/solver policy via PPO~\cite{schulman2017ppo} for software engineering.
Notably, negative-gradient destabilization is not unique to our setting: Kimina-Prover \cite{wang2025kiminaprover}, training a 72B Lean-4 prover via GRPO, reports early ``format collapse'' and resorts to probabilistically discarding negative-advantage samples ($\omega = 0.5$). We expect a related phenomenon---thin valid manifolds concentrating repulsion on a few off-manifold modes rather than spreading it across the vast invalid space---in formal proofs, verified specifications, and strictly typed code (\ref{app:manifold_collapse}). Instead of discarding samples, ANCORA addresses format collapse via manifold-constraining SFT and a seed-anchored frontier objective.

\textbf{Decoupled or Alternating Self-Play.}
An alternative family maintains separate proposer/solver models or updates them in alternating phases: Self-Questioning LMs \cite{chen2025selfquestioninglanguagemodels} (asymmetric self-play), Dr.Zero \cite{yue2026drzeroselfevolvingsearch} (alternating HRPO/GRPO on separate copies), SOAR \cite{sundaram2026soar} (bilevel teacher-student meta-RL), Tool-R0 \cite{acikgoz2026toolr0} (which finds shared weights \emph{hurt} their tool-calling domain---a finding our coupled spec/implementation setting reverses), STP \cite{wu2025stp} (alternating expert iteration for Lean), and Sol-Ver \cite{solver2025selfplay} (sequential SFT$+$DPO without RL proposing). These decoupled strategies sidestep gradient conflicts but forfeit the mutual regularization of shared parameters \cite{silver_mastering_2017}.

\textbf{Formal Verification and Self-Improvement.}
PSV \cite{wilf2025proposesolveverifyselfplay} uses an alternating propose-solve loop for Verus. However, its proposer relies on in-context learning rather than RL, and its solver is updated via rejection fine-tuning. DeepSeek-Prover \cite{ren2025deepseekproverv2advancingformalmathematical} uses Lean as a static reward model without generating new theorems. In contrast, ANCORA formulates specification generation and solution as a unified RL policy, trained jointly via MLRL-aligned GRPO, repurposing the compiler as a dynamic engine for autonomous problem discovery.

\textbf{Positioning.}
Each ingredient---GRPO, self-distillation, UCB curricula, and propose/solve training---has prior art, and AZR also jointly trains proposing and solving with Proposer rewards derived from Solver outcomes. ANCORA's contribution is the training mechanism obtained by combining these ideas under the requirements of a formal verifier: the same policy proposes formal specifications and solves them; training uses two nested groups (candidate specifications, and Solver attempts per specification); structural heuristics only filter invalid outputs and never provide partial reward; a proposed task enters the future seed pool only if it is novel and at least one solution verifies; and the Proposer is rewarded for a binary frontier event, a proposal solved by exactly one of $K$ Solver attempts. AZR is the closest unified-policy predecessor but stores execution-valid examples in growing buffers; PSV alternates proposal, verification, and rejection fine-tuning across iterations rather than updating both roles in the same online step. Table~\ref{tab:comparison} gives a per-feature comparison.

\begin{table}[t]
\centering
\caption{Feature comparison of self-play and self-improvement methods. \checkmark = present, $\times$ = absent.}
\label{tab:comparison}
\resizebox{\textwidth}{!}{%
\begin{tabular}{lccccc}
\toprule
Method & Unified Policy & RL Proposer & Manifold SFT & UCB Seed & Binary Reward \\
\midrule
AZR & \checkmark & \checkmark & $\times$ & $\times$ & $\times$ \\
SPELL & \checkmark & \checkmark & $\times$ & $\times$ & $\times$ \\
SPICE & \checkmark & \checkmark & $\times$ & $\times$ & $\times$ \\
Dr.Zero & $\times$ & \checkmark & $\times$ & $\times$ & $\times$ \\
PSV & $\times$ & $\times$ & $\times$ & $\times$ & \checkmark \\
\textbf{ANCORA} & \checkmark & \checkmark & \checkmark & \checkmark & \checkmark \\
\bottomrule
\end{tabular}}
\end{table}

\section{Preliminaries}
\label{sec:preliminaries}

\subsection{Maximum Likelihood Reinforcement Learning}
\label{sec:grpo_mlrl}
Standard binary-reward RL optimizes the expected pass rate $\mathbb{E}[p_\theta(x)]$, sharpening small-$k$ sampling but not necessarily expanding the underlying reasoning boundary \cite{yue2025doesreinforcementlearningreally}. MLRL \cite{tajwar2026maximumlikelihoodreinforcementlearning} instead targets the log-likelihood $\mathbb{E}[\log p_\theta(x)]$, equivalent at the gradient level to the weighted-$\textsc{pass}@k$ identity $\sum_{k\geq 1}\tfrac{1}{k}\nabla\textsc{pass}@k=\nabla\log p_\theta$. With $N$ samples and success count $m=\sum_i r_i$ ($r_i\in\{0,1\}$), the hit-conditioned finite-sample estimator normalizes by successes rather than total samples:
\begin{equation}
    \hat{g}_{\text{ML}} = \mathbbm{1}[m{>}0]\,\tfrac{1}{m}\sum_{i=1}^N r_i\,\nabla_\theta\log\pi_\theta(y_i|x).
    \label{eq:maxrl}
\end{equation}
Unconditionally this estimator is multiplied by the hit probability $1-(1-p_\theta(x))^N$; it is unbiased conditional on at least one success and becomes exact as the no-hit probability vanishes.
We implement Eq.~\ref{eq:maxrl} as a GRPO-style advantage with \emph{mean} (not std) normalization,
\begin{equation}
    A_i^{\mathrm{ML}} = \mathbbm{1}[\bar r{>}0]\,\frac{r_i-\bar r}{\bar r},\qquad \bar r=\tfrac{m}{N},
    \label{eq:mean_norm_mlrl_adv}
\end{equation}
because the resulting group-averaged gradient recovers the same positive-success coefficient as $\hat g_{\text{ML}}$ plus the corresponding centered score term within hit groups. A rare success ($m{=}1,N{=}8$) thereby weights $N/m{=}8\times$ REINFORCE's flat $1/N$, matching MLRL's $1/m$. Standard std-normalized GRPO instead gives $\mathcal{A}^+\!\propto\!1/\sqrt{m}$. In early exploratory runs, which were not a controlled comparison, we did not observe a large or consistent difference in final performance between the two; we use mean normalization because it recovers the $1/m$ weighting analytically, not because it has been shown to perform better. Throughout ANCORA, ``GRPO advantage'' refers to this mean-normalized, MLRL-aligned variant.

\subsection{Valid Manifold, Validity Collapse, and Diversity Contraction}
\label{sec:failure_modes}
We use $\mathcal{M}_G=\{y:G(y)=1\}$ as shorthand for the \emph{valid-output set} of a hard check $G$: the format and contract checks for the Proposer, or compiler verification for the Solver. Throughout the paper, ``manifold'' is used in this operational sense; we do not assume smooth geometric structure. We use \emph{validity collapse} (manifold collapse) for a sustained fall in the acceptance rate $p_G=\Pr_{y\sim\pi_\theta}[G(y)=1]$, and \emph{diversity contraction} when $p_G$ stays high but output entropy or diversity falls.

\textbf{Gate, do not grade.} In an all-fail group, binary rewards yield zero centered GRPO advantage. A nonconstant heuristic score instead ranks outputs for which the verifier supplied no success evidence, and that proxy gradient need not point toward verifier-accepted outputs. In exploratory Verus runs, tiered rewards for structural signals (closed reasoning or code blocks, required specification clauses, loop annotations) repeatedly worsened final performance and often reduced Proposer entropy. These runs were not a matched ablation suite, so we present ``gate, do not grade'' as an empirically motivated design principle rather than a general theorem: structural checks only decide eligibility, Solver success is determined by the verifier, and the Proposer reward depends only on the Solver's $K$ binary outcomes.

\textbf{Why gating still needs anchoring.} An $N$-sample group is all-fail with probability $(1-p_G)^N$, and the rare mixed groups give high-variance update directions that AdamW momentum can carry through later zero-signal groups (\ref{app:manifold_collapse}). With one positive and seven negatives, mean normalization assigns $+7$ to the positive and $-1$ to each negative: a contrastive update estimated from only seven sampled failures, not seven votes overwhelming one positive. These are exact starvation calculations plus a plausible interference mechanism, not a proof that validity must fall in every domain.

\section{Methodology}

\subsection{The ANCORA Optimization Framework}
ANCORA is a discrete-time dynamic system where a single unified policy $\pi_\theta$ serves as both Proposer and Solver. We refer to Proposer rollouts as \emph{question/specification generation} (Q-Gen) and Solver rollouts as \emph{solution generation} (S-Gen); correspondingly, subscripts $Q$ and $S$ denote the Proposer/Q-Gen and Solver/S-Gen branches. At each iteration, the Proposer samples a novel task $x' \sim \pi_\theta(x' | x_{\text{seed}})$ from a dynamically expanding Seed Pool $\mathcal{S}$, and the Solver attempts a solution $y \sim \pi_\theta(y | x')$ judged by a deterministic verifier $\mathcal{V}(x', y) \in \{0, 1\}$. The framework couples the two roles through a shared policy update while admitting new curriculum nodes only after strict filtering, novelty checks, and Solver verification, so that self-expansion remains anchored to the valid manifold.

\paragraph{Anchoring as a unifying principle}
ANCORA (\emph{Manifold-Anchored}) names a single rationale shared by iterative self-distilled SFT (\S\ref{sec:experiments}), the binary-only filter gate (\S\ref{sec:heuristic_reward}), and solved-only UCB-DAG admission (\S\ref{sec:ucb_dag}): refusing off-manifold gradient pressure under sparse verifier feedback, whether from external-distribution warm-up traces, heuristic-score noise, or unverified curriculum nodes. CoT itself is then the policy's manifold-preserving navigation through fragmented pretrained-known paths \cite{zhou2023limalessmorealignment, wen2025rlvrimplicitlyincentivizes}, the same property that lets RLVR's verifier signal generalize without overwriting the base model's pretrained prior.

\begin{figure}[t]
    \centering
    \includegraphics[width=0.95\linewidth]{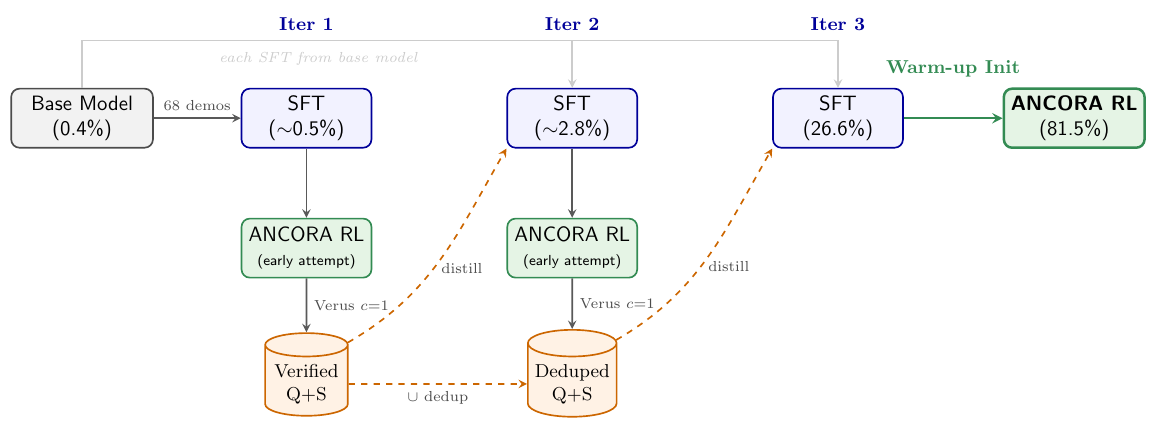}
    \caption{Iterative Self-Distilled SFT for valid manifold initialization. Direct RL from the base model fails due to verified-spec sparsity (positive-signal starvation). Early RL rollouts are filtered by the Verus compiler and distilled back into the policy; the iterative projection ($0.4\% \to 26.6\%$) anchors the model onto the valid manifold before full ANCORA RL.}
    \label{fig:sft_bootstrap}
\end{figure}

\subsection{Proposer Policy Strategy}

\subsubsection{Curriculum Generation}
\label{sec:adaptive_curriculum}
At each step, a seed $x_{\text{seed}}$ is selected from the DAG Seed Pool via UCB sampling (\S\ref{sec:ucb_dag}), and the Proposer generates a one-step curriculum extension. The iterative SFT warm-up teaches the model to make focused semantic changes rather than superficial rewrites, keeping exploration local enough to remain on the valid manifold while still walking toward new frontier regions; prompt details are deferred to \ref{appPrompts}.

\subsubsection{Filter Gate and Zero-Distortion Rewards}
\label{sec:heuristic_reward}
Generated candidates pass through a lightweight three-stage \emph{accept/reject gate} (format checks, compiler stub syntax validation, and semantic heuristics rejecting vacuous postconditions); rejected candidates are excluded from the GRPO group. We strictly decouple \emph{search} from \emph{learning}: heuristic scores and MinHash novelty~\cite{broder1997minhash} drive only gating and deduplication, while the policy gradient is driven exclusively by binary verifier outcomes $c_{j,k} \in \{0,1\}$. Injecting tiered heuristics into the reward instead induces spurious advantage rankings among invalid outputs and displaces correct-output likelihood mass~\cite{liu2025lld} (\ref{app:manifold_collapse}).

\subsubsection{Proposer Objective: Frontier-Targeting Reward}
\label{sec:proposer_obj}
The Proposer must generate specifications whose Solver outcomes are informative rather than degenerate: not saturated ($p_j \approx 1$), not purely off-manifold noise, but likely to expose useful MLRL signal for the Solver. Our contribution here is a \emph{frontier-targeting reward family}: bounded functions of the Solver's $K$ binary outcomes that turn the empirical outcome distribution itself into the Proposer's learning signal. For each filter-passing spec $x_j$, let $p_j = \tfrac{1}{K}\sum_k c_{j,k}$ be the Solver's empirical pass rate and $m_j = \sum_k c_{j,k}$ the success count. We instantiate the family with three designs:
\begin{equation}
    R_{\text{front}}(p_j,m_j) \in
    \left\{
    4p_j(1-p_j),\quad
    e^{-Kp_j}-e^{-K},\quad
    \mathbbm{1}[m_j=1]
    \right\}.
    \label{eq:entropy_reward_family}
\end{equation}
The Bernoulli-variance reward $4p_j(1-p_j)$ is a smooth symmetric uncertainty kernel (equivalent up to a constant to the Gini impurity, and the second-order Taylor expansion of $H(p)$ around $p{=}1/2$), peaked at $p_j{=}1/2$ and zero at both extremes. The exponential variant is a low-pass stress-test kernel rather than an entropy proxy: it upweights very low empirical pass rates and therefore relies on the filter and solved-only admission gates to avoid off-manifold zero-signal proposals. The Band-1-of-$K$ variant is a sparse binary nonzero-signal proxy:
\begin{equation}
    R_{\text{prop}}(x_j) = \mathbbm{1}\!\left[m_j = 1\right].
    \label{eq:proposer_reward}
\end{equation}
Our main run uses Band-1-of-$K$: a sparse binary event ($m{=}1$ out of $K$) that emphasizes hard-but-nonzero specifications and avoids sensitivity to noisy continuous pass-rate estimates. This $m{=}1$ regime is precisely where MLRL's $1/m$ coefficient (Eq.~\ref{eq:maxrl}) diverges most from REINFORCE's flat $1/N$, so Band-1-of-$K$ is the binary indicator of the rare-success region that the Solver's MLRL update already treats as most informative---rewarding the Proposer for landing there directly aligns Q/S likelihood signals. The other two shapes, ablated in \ref{sec:reward_ablation}, are equally stabilized by the same SFT$+$UCB-DAG mechanisms; we treat the family itself, rather than a single shape, as the methodological contribution.

\subsection{Solver Policy Strategy}
\label{sec:solver_obj}
The solver is optimized under the same MLRL principle. For each spec $x_j$ with $K$ attempts and group mean $\bar{c}_j = \tfrac{1}{K}\sum_k c_{j,k}$, the token advantage on solution $y_{j,k}$ is
\begin{equation}
    \mathcal{A}_{\text{solve}}(y_{j,k}) = \frac{c_{j,k} - \bar{c}_j}{\bar{c}_j + \epsilon},
    \label{eq:solver_adv}
\end{equation}
By Eq.~\ref{eq:maxrl}, the group-averaged gradient induced by this advantage is the finite-sample MLRL/MaxRL estimator plus a zero-mean score baseline. All-fail and all-pass groups produce zero centered gradient; mixed groups update the Solver on exactly the uncertain cases where likelihood mass can still move.

\subsection{UCB-Guided Seed Pool and Curriculum DAG}
\label{sec:ucb_dag}

A single Proposer rollout usually produces only a local variation of its conditioning seed. To reach distant regions without leaving the valid manifold, ANCORA composes such local moves: verified novel specs are reinserted as future seeds, forming a curriculum DAG (Figure~\ref{fig:ucb_dag}). Initial curated seeds are protected roots; non-original nodes are admitted only if at least one Solver rollout verifies them (\emph{solved-only admission}, Lemma~\ref{lem:solved_only}); near-duplicates are rejected by MinHash. Seed selection uses MCTS-style UCB descent~\cite{auer2002finitetime,kocsis2006bandit} with a tunable \emph{root quota} $\rho\in[0,1]$: a fraction $\rho$ of the batch is sampled from UCB-chosen roots, and the remaining $1-\rho$ performs UCB descent into Q-Gen-discovered children. Setting $\rho{=}1$ restricts seeding to the curated roots; the reported run conditions its proposals on the curated roots. Discounted backups, depth-capped absorb-and-merge, and production deviations are formalized in \ref{subsec:discounted_ucb_model}--\ref{subsec:production_deviations}.

\begin{figure}[t]
    \centering
    \includegraphics[width=\linewidth]{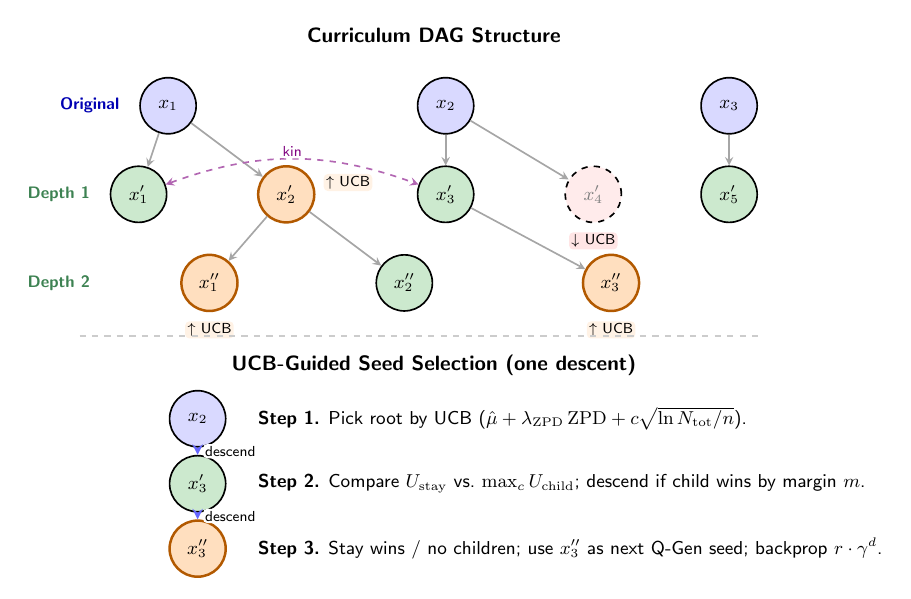}
    \caption{The UCB-guided Curriculum DAG. \textbf{Top:} Original seeds (blue) are protected roots; Proposer-generated specs (green) are admitted as children only after solved-only verification; high-UCB frontier nodes (orange) become next-step seeds, low-UCB nodes (red, dashed) are deprioritized; dashed-violet edges mark structural kinship (Jaccard $\geq 0.70$). \textbf{Bottom:} Selection first picks a root by UCB, then performs local stay-vs-child descent, using the terminal node as the next Q-Gen conditioning seed. \ref{subsec:production_deviations} gives the full production deviations.}
    \label{fig:ucb_dag}
\end{figure}

\paragraph{Why composition produces novel problems}
A single Q-Gen step from a fixed seed explores only a local neighborhood; the DAG composes such steps. Let $\mathcal{F}_k$ be the history, $Z_k\in\mathcal{R}_k$ the selected seed, $X_k\sim\pi_\theta(\cdot|Z_k)$, and $g(X_k)\in\{0,1\}$ the admission gate, so $\mathcal{R}_{k+1}=\mathcal{R}_k\cup\{X_k:g(X_k){=}1\}$.
\begin{proposition}[Compositional reachability]\label{prop:reachability}
Before saturation, assume $\Pr[g(X_k){=}1,X_k\notin\mathcal{R}_k\mid\mathcal{F}_k]\geq\delta>0$ for all $k$. Then $\mathbb{E}[|\mathcal{R}_{k+1}|\mid\mathcal{F}_k]\geq|\mathcal{R}_k|+\delta$ and $|\mathcal{R}_k|\to\infty$ a.s. (until saturation if finite). Let $\mathcal{U}_1=\cup_{s\in\mathcal{S}_0}\{x:g(x){=}1,\pi_\theta(x|s)>0\}$. If an admitted non-root $z$ is revisited with probability $\alpha>0$ and $\Pr[g(X){=}1,X\notin\mathcal{U}_1|Z{=}z]\geq\beta>0$, then a composition-only spec outside $\mathcal{U}_1$ is eventually admitted a.s.
\end{proposition}
\noindent\textit{Proof.} Let $I_k=\mathbbm{1}[g(X_k){=}1,X_k\notin\mathcal{R}_k]$. Since $|\mathcal{R}_{k+1}|=|\mathcal{R}_k|+I_k$, the expectation bound follows. Also $\Pr[I_k=\cdots=I_{k+m-1}=0\mid\mathcal{F}_k]\leq(1-\delta)^m$, giving finite waiting time for each new node a.s.; induction gives divergence until saturation. The second claim is the same geometric-tail argument with success probability $\alpha\beta$. $\square$

\noindent Thus sustained expansion needs persistent exploration, not only a good local generator. ANCORA approximates this with protected roots, a root/DAG mixture ($\rho<1$), UCB bonuses, and solved-only admission; a literal guarantee would add an $\epsilon$-refresh or sampling floor over active roots/frontier leaves. Empirically, Q-Gen validity stays bounded away from zero and the seed pool grows monotonically over the logged diagnostic window (\ref{app:additional_ablations}).

\subsection{Coupled Joint Optimization}
\label{sec:joint_opt}

Both Proposer and Solver are segments of a single unified policy $\pi_\theta$. For each iteration, $B$ prompts yield $N$ candidate specifications each; each valid specification receives $K$ solver rollouts. Problem tokens receive $R_{\text{prop}}$ via the mean-normalized GRPO variant in \S\ref{sec:grpo_mlrl}; Proposer advantages are normalized within each seed's $N$ Q-Gen candidates, with invalid candidates assigned zero Proposer reward and skipped only by the Solver rollout. Solver tokens receive binary $R_{\text{solve}} \in \{0,1\}$ over $K$ solutions, for which the same normalization realizes the MLRL weighted-$\textsc{pass}@k$ estimator exactly. The Band-1-of-$K$ Proposer reward is also binary and follows the same algebra, but targets the likelihood of an \emph{informative-frontier event} ($m_j{=}1$ of $K$) rather than literal Solver correctness---an MLRL-style update on a different binary signal.

\textbf{Branch balancing.} Solver responses are much longer than Proposer specs, so raw token-level advantages let the Solver dominate. We equalize generated token mass per batch with $\mathrm{scale}_Q = (T_Q+T_S)/(2T_Q)$ and $\mathrm{scale}_S = (T_Q+T_S)/(2T_S)$, then weight the Proposer branch by a static scale $q_{\mathrm{scale}}$ or by adaptive $\ell_2$-balancing against the Solver branch; both settings and the reverse-imbalance failure mode are discussed in \ref{sec:reverse_imbalance} and~\ref{sec:branch_weight_ablation}; Algorithm~\ref{alg:gsv_loop} gives the full iteration.

\begin{algorithm}[!t]
    \footnotesize
    \caption{ANCORA: One Training Iteration}
    \label{alg:gsv_loop}
    \begin{algorithmic}[1]
        \STATE \textbf{Input}: Policy $\pi_\theta$, Verifier $\mathcal{V}$, Seed Pool DAG $\mathcal{S}$, Batch $B$, Specs/prompt $N$, \mbox{Solutions/spec} $K$
        \STATE \textbf{// Stage 1: UCB Seed Selection}
        \STATE Select $B$ seeds from $\mathcal{S}$ via UCB tree search (\S\ref{sec:ucb_dag}): roots by $\hat\mu_i + c\sqrt{\ln N_{\text{tot}} / n_i}$, then DAG descent with root quota $\rho$. The exploit term $\hat\mu_i$ is a discounted empirical mean or a max-ratcheted aggregate (\ref{sec:empirical_ucb_negative}).
        \STATE \textbf{// Stage 2: Question Generation (Q-Gen; Proposer Rollout)}
        \FOR{each seed $x_{\text{seed}}$}
            \STATE Construct 1-shot prompt from $x_{\text{seed}}$; sample $N$ candidate specs $\{x_1, \dots, x_N\} \sim \pi_\theta$
            \STATE Run 3-stage filter (structural $\to$ compiler stub $\to$ semantic heuristics) on each $x_j$
            \STATE \textbf{Discard} rejected candidates \COMMENT{No reward signal; excluded from GRPO group}
        \ENDFOR
        \STATE \textbf{// Stage 3: Solution Generation (S-Gen; Solver Rollout)}
        \FOR{each valid spec $x_j$}
            \STATE Sample $K$ solutions $\{y_{j,1}, \dots, y_{j,K}\} \sim \pi_\theta(\cdot \mid x_j)$; verify $c_{j,k} = \mathcal{V}(x_j, y_{j,k}) \in \{0,1\}$
        \ENDFOR
        \STATE \textbf{// Stage 4: Coupled MLRL Advantage Computation}
        \STATE \textit{Proposer:} $m_j = \sum_k c_{j,k}$; \; $R_{\text{prop}}(x_j) = \mathbbm{1}[m_j = 1]$ \COMMENT{Eq.~\ref{eq:proposer_reward}; Band-1-of-$K$}
        \STATE \hspace{3.1em} $\bar{R}_{\text{prop}} = \frac{1}{|\mathcal{G}_q|}\sum_{x \in \mathcal{G}_q} R_{\text{prop}}(x)$; \; $\mathcal{A}_{\text{prop}}(x_j) = \frac{R_{\text{prop}}(x_j)-\bar{R}_{\text{prop}}}{\bar{R}_{\text{prop}}+\epsilon}$ \COMMENT{Mean-normalized among valid specs}
        \STATE \textit{Solver:} $\mathcal{A}_{\text{solve}}(y_{j,k}) = \frac{c_{j,k} - \bar{c}_j}{\bar{c}_j + \epsilon}$ \COMMENT{Mean-normalized MLRL estimator}
        \STATE \textbf{// Stage 5: Joint Policy Update}
        \STATE Broadcast $\mathcal{A}_{\text{prop}}$ and $\mathcal{A}_{\text{solve}}$ to proposer/solver response tokens
        \STATE Compute generated token counts $T_Q,T_S$; rescale by $\mathrm{scale}_Q=(T_Q+T_S)/(2T_Q)$ and $\mathrm{scale}_S=(T_Q+T_S)/(2T_S)$
        \STATE Apply branch weights $w_Q$ (static or $\ell_2$-balanced; \ref{sec:branch_weight_ablation}) and $w_S$; single gradient step on $\pi_\theta$
        \STATE \textbf{// Stage 6: Seed Pool Absorption \& Pruning}
        \STATE \textbf{Absorb:} Add novel verified specs as children of their seed parent (linkage expansion for diversity)
        \STATE \textbf{Prune:} Deduplicate near-duplicate nodes (MinHash); an optional long-horizon cleanup of dead-end nodes is available but not triggered in our main run (\ref{subsec:prune_demoted}). Backpropagate reward along UCB path
\end{algorithmic}
\end{algorithm}

\section{Experiments}
\label{sec:experiments}

In our Verus instantiation, specifications (\texttt{spec}/\texttt{ensures}) serve as questions and implementations as answers, with the compiler providing deterministic binary feedback.

\subsection{Setup}

\paragraph{Datasets} We evaluate on three Verus benchmarks: \textbf{Dafny2Verus} (274 problems, translated from Dafny \cite{aggarwal_alphaverus_2024}; of these, 234 pass the seed-pool quality filter and form the curriculum's initial roots in the Transfer setting), \textbf{MBPP}-Verified (78 problems \cite{austin2021programsynthesislargelanguage, yangAutoVerusAutomatedProof2025}), and \textbf{HumanEval}-Verified (85 functions converted from HumanEval \cite{chen2021evaluating}). We adopt Pass@$k$ \cite{chen2021evaluating} with $n{=}100$ samples. Evaluation uses the filtered splits released with PSV (241 / 74 / 73 tasks).

\paragraph{Baselines} We compare against: (1)~AlphaVerus \cite{aggarwal_alphaverus_2024} (50-shot prompting), the prior SOTA for verified code generation; (2)~iterative RFT \cite{yuan2023scalingrelationshiplearningmathematical, singhHumanDataScaling2024}, expert iteration without proposing; and (3)~PSV \cite{wilf2025proposesolveverifyselfplay}, alternating propose-solve with rejection fine-tuning. PSV is the only directly comparable trained self-play baseline evaluated with the same base model, verifier, and benchmark splits; AlphaVerus and RFT provide reference points but do not isolate self-play. Self-Questioning Language Models \cite{chen2025selfquestioninglanguagemodels} and Dr.~Zero \cite{yue2026drzeroselfevolvingsearch} do not report Verus results with the same base model, verifier, data exposure, and training budget; adapting them would change several factors at once, so we do not report reimplementations.

\paragraph{Implementation} Built on \texttt{verl}/HybridFlow \cite{sheng2024hybridflow} with \texttt{vllm} for rollout \cite{kwon2023efficient}. All experiments use 2$\times$A100 (40GB). We perform full-parameter fine-tuning on Qwen2.5-Coder-3B-Instruct \cite{hui2024qwen25coder} with AdamW, Flash-Attention 2, and bfloat16. The SFT phase uses $\text{lr}{=}1{\times}10^{-5}$ with cosine schedule; the RL phase uses $\text{lr}{=}1{\times}10^{-6}$ with 5\% linear warmup. Each training step samples $B{=}8$ seeds from the curriculum DAG; each seed yields $N{=}8$ candidate specifications, and each valid specification receives $K{=}8$ solver rollouts. Maximum sequence length is 4096 tokens. The final RL stage takes approximately 200 GPU-hours on 2$\times$A100; Section~\ref{sec:compute} gives the cost of the full pipeline.

\paragraph{Iterative SFT warm-up}
The unadapted Qwen2.5-Coder-3B-Instruct achieves only 0.4\% pass@1 on Dafny2Verus, and direct RL from this point triggers the manifold collapse of \ref{app:manifold_collapse}. We start with 68 basic Verus specs whose Q-Gen/S-Gen traces are generated by Kimi-K2.5 \cite{kimi2026k25} and DeepSeek-V3.2 \cite{deepseekai2025deepseekv32}, mainly to teach formatting and proof skeletons ($0.4\%\!\to\!{\sim}0.5\%$), then distill MinHash-deduplicated Verus-verified rollouts across three early ANCORA iterations (2 epochs); Table~\ref{tab:sft_bootstrap} shows the progression. Three iterations were the empirical stopping point of our pipeline rather than a tuned optimum: one exploratory fourth stage, distilled from a further batch of verifier-accepted RL traces, did not exceed the third stage, but a single attempt cannot distinguish overfitting from saturation or run-to-run variation. The jump from iteration~2 to iteration~3 coincides with a change in data: iteration~3 trains for two epochs on all 16,635 deduplicated, verifier-accepted Proposer/Solver examples accumulated from earlier rounds rather than on the newest batch only, so data amount, composition, and stage number changed together. In the TTT setting all methods already use evaluation specifications as RL seed prompts.

\begin{table}[h]
\centering
\small
\caption{SFT bootstrapping via accumulated verified rollouts. Each iteration collects Verus-verified Proposer/Q-Gen specifications and Solver/S-Gen implementations from ANCORA training, deduplicates them, and distills the resulting traces into SFT data for the next round.}
\label{tab:sft_bootstrap}
\begin{tabular}{llc}
\toprule
Phase & Data source & D2V p@1 \\
\midrule
Base model & --- & 0.4 \\
SFT iter\,1 & 68 basic seed demos (Kimi/DS) & ${\sim}$0.5 \\
SFT iter\,2 & + iter\,1 verified rollouts & ${\sim}$2.8 \\
SFT iter\,3 & + dedup(iter\,1{+}2 rollouts) & 26.6 \\
\midrule
\textbf{+ ANCORA RL} & --- & \textbf{81.5} \\
\bottomrule
\end{tabular}
\end{table}

The final \texttt{sft\_init} reaches 26.6\% / 8.8\% / 6.2\% pass@1 on D2V / MBPP / HumanEval---sufficient manifold coverage for stable RL training (\ref{app:manifold_collapse}).

\paragraph{Data provenance and overlap} For 75 Solver rows of the final SFT data, the teacher models wrote only the reasoning text, while the implementation bodies were taken from a pre-existing pool of verified solutions whose origin we cannot fully reconstruct; we therefore do not claim that the pipeline uses no human-written solutions. Matching the final 16,635-row SFT data against the evaluation sets after whitespace normalization, 118 D2V and 3 MBPP evaluation specifications appear as SFT prompts---expected in TTT, where rollouts on evaluation seeds are distilled---and for 4 D2V tasks the reference implementation appears as an SFT target. Excluding these tasks changes D2V TTT pass@1 from 81.46\% to 81.33\% and MBPP TTT pass@1 from 44.11\% to 42.61\%. Exact matching does not detect near-duplicates.

\subsection{Main Results}

Table~\ref{tab:main_results} presents pass rates across all benchmarks.

\begin{table*}[t]
\centering
\caption{Pass@$k$ (\%) on Verus benchmarks. All methods use Qwen2.5-Coder-3B-Instruct. Transfer Learning trains on Dafny2Verus and evaluates on MBPP/HumanEval as held-out sets; Test-Time Training (TTT) uses filtered evaluation roots as RL seed prompts. Our TTT run uses a single root pool containing the benchmark specifications and reports each benchmark separately. The PSV-format 1-shot row evaluates the same checkpoint with PSV's message construction and decoding (Section~\ref{sec:passk}).}
\label{tab:main_results}
\resizebox{\textwidth}{!}{
\begin{tabular}{lccccccccc}
\toprule
 & \multicolumn{3}{c}{\textbf{Dafny2Verus}} & \multicolumn{3}{c}{\textbf{MBPP}} & \multicolumn{3}{c}{\textbf{HumanEval}} \\
 \cmidrule(lr){2-4} \cmidrule(lr){5-7} \cmidrule(lr){8-10}
 & p@1 & p@5 & p@10 & p@1 & p@5 & p@10 & p@1 & p@5 & p@10 \\
\midrule
\multicolumn{10}{l}{\emph{Transfer Learning} (seed $X_0$ = Dafny2Verus)} \\
\midrule
AlphaVerus (50-shot) & 24.06 & 52.42 & 63.44 & 6.48 & 18.36 & 24.57 & 7.24 & 15.38 & 18.02 \\
RFT (1-shot) & -- & -- & -- & 10.99 & 26.03 & 31.19 & 10.99 & 18.07 & 20.02 \\
PSV (1-shot) & -- & -- & -- & 25.25 & 38.32 & \textbf{41.51} & 16.18 & \textbf{21.61} & \textbf{23.13} \\
\textbf{ANCORA (0-shot)} & -- & -- & -- & \textbf{36.20} & \textbf{38.49} & 39.94 & \textbf{17.19} & 20.16 & 21.03 \\
\midrule
\multicolumn{10}{l}{\emph{Test-Time Training} (seed $X_0$ = test dataset)} \\
\midrule
AlphaVerus (50-shot) & 24.06 & 52.42 & 63.44 & 6.48 & 18.36 & 24.57 & 7.24 & 15.38 & 18.02 \\
RFT (1-shot) & 34.46 & 57.12 & 63.40 & 3.83 & 14.56 & 22.64 & 5.56 & 12.46 & 14.80 \\
PSV (1-shot) & 65.63 & 78.04 & 80.06 & 36.78 & \textbf{51.22} & \textbf{53.67} & 19.07 & 23.42 & 25.16 \\
\textbf{ANCORA (0-shot)} & \textbf{81.46} & 83.86 & 84.58 & \textbf{44.11} & 45.62 & 45.91 & \textbf{19.15} & 22.67 & 23.91 \\
ANCORA (PSV-format 1-shot) & 74.28 & \textbf{84.32} & \textbf{85.42} & 41.15 & 46.06 & 46.20 & 18.67 & \textbf{26.74} & \textbf{28.58} \\
\bottomrule
\end{tabular}
}
{\footnotesize AlphaVerus uses 50-shot prompting (no training); PSV and RFT results are from \cite{wilf2025proposesolveverifyselfplay}. ANCORA rows come from a single training run; the Transfer values are single-run point estimates, and three MBPP prompts also appear in the warm-up data (Section~\ref{sec:experiments}).}
\end{table*}

\paragraph{Evaluation protocol note}
ANCORA is reported at \emph{0-shot}, while PSV results are at 1-shot \cite{wilf2025proposesolveverifyselfplay}. Because an exemplar can help or hurt either method, the 0-shot comparison is not conservative in a known direction; we therefore also evaluate ANCORA under PSV's 1-shot protocol (Section~\ref{sec:passk}).

ANCORA outperforms baselines on Dafny2Verus and MBPP, and remains competitive on HumanEval. At the reported checkpoint (with $n{=}100$ rollouts), ANCORA achieves 81.5\% pass@1 on Dafny2Verus, a $+15.8$-point gain over PSV (1-shot), and 74.3\% under PSV's own 1-shot protocol. In TTT on MBPP, ANCORA reaches 44.1\% pass@1, exceeding PSV by 7.3 points. On HumanEval, ANCORA matches PSV at pass@1 (19.2\% vs.\ 19.1\%) and remains within 1.3 points on pass@10 (23.9\% vs.\ 25.2\%), reflecting the protocol-controlled tradeoff: ANCORA recovers PSV's 1-shot pass@1 without any in-context exemplar at inference time. The flat TTT MBPP pass@$k$ curve (44.1$\to$45.9 across $k{=}1{\to}10$) indicates the policy has sharpened to near-deterministic solutions for solvable problems. On Transfer MBPP, ANCORA's pass@10 (39.9) is within 1.6 points of PSV (41.5), and its pass@1 remains higher (36.2 vs.\ 25.3 for PSV) despite a slightly narrower sampling tail. The narrower HumanEval margin reflects task-origin domain shift orthogonal to this sharpening: although all three benchmarks are evaluated as Verus code, HumanEval-Verified's 85 functions are Verus translations of general Python programming tasks (string parsing, list manipulation, arithmetic loops) rather than algorithm-verification specifications native to D2V, so the D2V-seeded Proposer distribution transfers less directly to HumanEval than to MBPP's algorithmic-task style. This is the price of training the Proposer on a single seed family---transferable problem-solving structure is preserved (HE pass@1 rises from 6.2\% SFT init to 17.2\% Transfer / 19.2\% TTT), but the gain narrows as the evaluation domain drifts further from the seed-induced distribution. With a single ANCORA run and 73 HumanEval tasks, the HumanEval margins in both settings are too small to interpret, and we do not claim superiority over PSV on HumanEval.

\subsection{Prompt-Matched Evaluation and Sampling Behaviour}
\label{sec:passk}

We evaluated the reported checkpoint, without retraining, with PSV's documented 1-shot message construction: one exemplar drawn independently for each rollout from the public 241-pair D2V gold pool (excluding the task itself on D2V), temperature 0.8, top-$p$ 0.9, and 100 samples per task. Under this protocol ANCORA still exceeds the reported 1-shot PSV pass@1 on D2V (74.28\% vs.\ 65.63\%) and MBPP (41.15\% vs.\ 36.78\%), so the main pass@1 gains are not explained by the 0-shot/1-shot mismatch. Relative to the 0-shot values, adding the exemplar lowers pass@1 on all three datasets; it raises pass@5 and pass@10 on all three.

The per-task outcome distribution explains the shape of the pass@$k$ curves. Under 0-shot prompting the 74 MBPP tasks split sharply: 40 were never solved, 30 were solved in all 100 samples, and only 4 in between. With the exemplar, the number of HumanEval tasks solved in some but not all samples grows from 11 to 23 and the number solved in every sample falls from 9 to 0; on MBPP the intermediate group grows from 4 to 35. The exemplar thus slightly reduces single-sample success but widens the spread of outcomes across samples. Because the exemplar is resampled across rollouts, this widening reflects both prompt variation and completion sampling and does not measure semantic diversity. Against PSV under 1-shot prompting the pass@$k$ picture is mixed: ANCORA is lower on MBPP pass@5/10 (46.06/46.20 vs.\ 51.22/53.67) and higher on HumanEval (26.74/28.58 vs.\ 23.42/25.16). We therefore treat sampling behaviour as a property of both the checkpoint and the prompt, and make no claim that ANCORA reduces sample complexity.

\paragraph{Uncertainty} The headline numbers come from a single training run. Normal-approximation intervals over tasks for the reported 0-shot pass@1 are $[76.7, 86.2]$ on D2V, $[32.9, 55.3]$ on MBPP, and $[10.5, 27.8]$ on HumanEval; they describe variation across benchmark tasks, not uncertainty due to retraining. The small-sample live evaluation logged every 100 steps between steps 1,999 and 2,499 gives 80.8, 80.5, 82.0, 80.5, 81.8, and 80.7\% D2V pass@1, so the reported value is not an isolated peak. We report no training-seed variance and do not describe the differences to PSV as statistically significant.

\subsection{Learning Curves and Coupled-Branch Ablation}
\label{sec:ablation}

Figure~\ref{fig:freeze_ablation} traces live pass@1 in the test-time-training setting corresponding to the lower block of Table~\ref{tab:main_results}. Starting from the manifold-projected \texttt{sft\_init} (26.6\% / 8.8\% / 6.2\% on D2V / MBPP / HumanEval), the final $n{=}100$ evaluation gives the headline numbers in Table~\ref{tab:main_results}. Dafny2Verus and MBPP show a sustained upward trend with small checkpoint-level fluctuations; HumanEval shows higher variance due to its smaller evaluation set (73 tasks). Late-stage Solver responses are well-formed proof-oriented derivations rather than pathological repetition (\ref{appExamples}).

To isolate the contribution of each branch in the coupled optimization, we conduct \emph{branch-freeze} experiments: \textbf{Freeze-Q} masks the Proposer/Q-Gen loss contribution (only Solver tokens update the shared policy), and \textbf{Freeze-S} masks the Solver/S-Gen loss contribution (only Proposer tokens update it). Both variants start from the same SFT initialization as the full ANCORA run. Figure~\ref{fig:freeze_ablation} reports pass@1 on all three benchmarks.

\begin{figure}[t]
\centering
\begin{tikzpicture}
\begin{groupplot}[
    group style={
        group size=3 by 1,
        horizontal sep=0.55cm,
    },
    width=0.36\textwidth,
    height=0.32\textwidth,
    xlabel={Step},
    tick label style={font=\scriptsize},
    label style={font=\small},
    title style={font=\small\bfseries},
    grid=major,
    grid style={gray!20},
    xmin=0, xmax=2050,
]
\nextgroupplot[title={(a) Dafny2Verus}, ylabel={Pass@1 (\%)}, ymin=0, ymax=85]
\addplot[blue, thick, mark=*, mark size=1.2pt] coordinates {
    (0,26.64) (99,38.59) (199,52.03) (299,57.84) (399,60.71) (499,62.57)
    (599,65.44) (699,68.05) (799,71.41) (899,72.37) (999,70.75)
    (1099,72.28) (1199,74.07) (1299,75.19) (1399,77.01) (1499,78.51)
    (1599,78.05) (1699,74.98) (1799,79.42) (1899,79.54) (1999,80.83)
};
\label{plot:ancora}
\addplot[orange, thick, mark=square*, mark size=1.2pt] coordinates {
    (0,26.64) (99,36.89) (199,51.91) (299,52.03) (399,53.20) (499,48.13)
};
\label{plot:freezeq}
\addplot[red!70!black, thick, mark=triangle*, mark size=1.2pt] coordinates {
    (0,26.64) (99,30.83) (199,30.25) (299,31.54) (399,16.27) (499,10.11)
};
\label{plot:freezes}
\addplot[violet!70!black, thick, mark=diamond*, mark size=1.2pt, dashed] coordinates {
    (0,26.64) (99,39.34) (199,45.39) (299,46.80) (399,51.29) (499,51.20)
    (599,33.36) (699,21.91) (799,11.70) (899,26.89) (999,16.27)
    (1099,17.84) (1199,19.75) (1299,18.67) (1399,22.32) (1499,22.32)
    (1599,19.50) (1699,20.50) (1799,18.34) (1899,14.27) (1999,11.95)
};
\label{plot:nodescent}

\nextgroupplot[title={(b) MBPP}, ymin=0, ymax=50]
\addplot[blue, thick, mark=*, mark size=1.2pt] coordinates {
    (0,8.78) (99,16.62) (199,26.49) (299,31.76) (399,31.76) (499,31.76)
    (599,35.14) (699,35.00) (799,36.22) (899,38.92) (999,41.08)
    (1099,42.30) (1199,41.89) (1299,42.97) (1399,41.35) (1499,43.11)
    (1599,43.24) (1699,43.24) (1799,42.97) (1899,45.54) (1999,44.73)
};
\addplot[orange, thick, mark=square*, mark size=1.2pt] coordinates {
    (0,8.78) (99,16.76) (199,28.24) (299,29.05) (399,29.86)  (499,28.16)
};
\addplot[red!70!black, thick, mark=triangle*, mark size=1.2pt] coordinates {
    (0,8.78) (99,12.70) (199,10.00) (299,7.97) (399,2.16) (499,2.01)
};
\addplot[violet!70!black, thick, mark=diamond*, mark size=1.2pt, dashed] coordinates {
    (0,8.78) (99,14.86) (199,13.24) (299,20.54) (399,24.05) (499,19.73)
    (599,3.78) (699,0.27) (799,1.62) (899,4.32) (999,1.35)
    (1099,5.41) (1199,5.68) (1299,5.68) (1399,8.92) (1499,8.92)
    (1599,4.05) (1699,4.05) (1799,4.05) (1899,2.70) (1999,0.00)
};

\nextgroupplot[title={(c) HumanEval}, ymin=0, ymax=25]
\addplot[blue, thick, mark=*, mark size=1.2pt] coordinates {
    (0,6.16) (99,8.63) (199,10.68) (299,15.34) (399,13.97) (499,13.01)
    (599,16.44) (699,13.42) (799,16.30) (899,14.52) (999,17.67)
    (1099,18.63) (1199,14.25) (1299,16.16) (1399,15.89) (1499,18.77)
    (1599,19.04) (1699,20.41) (1799,20.14) (1899,16.03) (1999,17.26)
};
\addplot[orange, thick, mark=square*, mark size=1.2pt] coordinates {
    (0,6.16) (99,8.77) (199,10.55) (299,10.82) (399,9.73) (499,8.71)
};
\addplot[red!70!black, thick, mark=triangle*, mark size=1.2pt] coordinates {
    (0,6.16) (99,5.89) (199,7.81) (299,7.40) (399,3.01) (499,2.16)
};
\addplot[violet!70!black, thick, mark=diamond*, mark size=1.2pt, dashed] coordinates {
    (0,6.16) (99,12.05) (199,7.95) (299,8.22) (399,7.67) (499,7.67)
    (599,3.56) (699,2.47) (799,0.00) (899,7.12) (999,1.92)
    (1099,0.55) (1199,3.56) (1299,2.74) (1399,1.37) (1499,1.92)
    (1599,5.21) (1699,5.48) (1799,3.56) (1899,2.74) (1999,2.74)
};

\end{groupplot}
\node[anchor=north, yshift=-2mm] at (current bounding box.south) {%
    \scriptsize
    \ref{plot:ancora}\,ANCORA\qquad
    \ref{plot:freezeq}\,Freeze-Q\qquad
    \ref{plot:freezes}\,Freeze-S\qquad
    \ref{plot:nodescent}\,Root-only (other config.)
};
\end{tikzpicture}
\caption{Branch-freeze runs and a root-only run with a different configuration in the test-time-training setting across $\sim$2{,}000 RL steps (one run each). Curves show live training-time evaluation (small-$n$); the headline pass@$k$ numbers in Table~\ref{tab:main_results} use $n{=}100$ at the final checkpoint and may differ by 1--2 points due to sampling variance, particularly on HumanEval (73 tasks). See text for detailed analysis of the failure modes of Freeze-Q (plateaus), Freeze-S (collapses), and the root-only run (initially peaks, then collapses).}
\label{fig:freeze_ablation}
\end{figure}

\begin{itemize}[leftmargin=1.2em,itemsep=2pt,topsep=2pt]
\item \textbf{Freeze-Q (static curriculum):} initially tracks ANCORA (step~99: 36.9\% vs.\ 38.6\% D2V) but plateaus by step~300, ending 14.4 points below ANCORA at the matched step~499---suggesting that Proposer co-evolution matters.
\item \textbf{Freeze-S (reward collapse):} proposed problems become unsolvable or trivial, the informative frontier-reward region vanishes, and performance falls below SFT init by step~400---a manifestation of the manifold collapse in \ref{app:manifold_collapse}.
\item \textbf{Root-only run with a different configuration (peak, then collapse):} a run that seeds only from the curated roots and differs from the main run in its Proposer reward and UCB statistic initially \emph{climbs steeply} (51.3\% D2V at step~400), then collapses after step~500 (D2V $12\%$ by step~1{,}999; MBPP falls to $1.6\%$ by step~799 and $0\%$ by step~1{,}999). Because several settings differ, we do not attribute this collapse to a single component.
\end{itemize}
Each curve comes from a single run, so we read the differences as qualitatively distinct trajectories rather than precise effect sizes.

\paragraph{Proposer entropy in the main run} The main run also shows diversity contraction: Proposer entropy decreases from 1.443 at step~0 to 0.947 at step~528, 0.486 at step~1{,}999, and 0.232 at step~2{,}245. It does not reach the near-zero level of the Ratchet-NoXO cell (0.046 at step~528; \ref{sec:empirical_ucb_negative}), and the model continues to generate valid outputs and to grow the task pool. Reduced exploration may contribute to the late plateau, but this trajectory alone does not establish the cause; because the runs differ in several settings, we do not attribute the difference to a single component.

\subsubsection{Additional Diagnostics}

Detailed training diagnostics and partial-control ablations are reported in \ref{app:additional_ablations}, including Q-Gen validity, seed-pool growth, Proposer reward variants, UCB exploit statistics, ZPD admission-window clipping (a negative-result variant of the ZPD prior), and static vs.\ adaptive Q/S branch weighting.

\subsection{Compute}
\label{sec:compute}
The approximately 200 GPU-hours above refer to the final RL stage. Using the same step-based estimate, the earlier 300- and 600-step bootstrap self-play stages add about 30 and 60 A100-hours, iterative SFT about 10, and the final evaluation with 100 samples per task about 3, so the successful end-to-end path takes approximately 300 A100-hours. These are allocated GPU hours, not hours of continuous model updating: every online step waits for task generation, multiple Solver attempts, and verifier feedback, and long generations and verifier latency dominated the late stage. Failed exploratory runs and ablations are research overhead that was not recorded under a consistent accounting protocol and is not included.

\section{Discussion and Conclusion}
\label{sec:conclusion}

We presented ANCORA, a framework for autonomous self-improvement in verifiable reasoning that unifies problem proposing and solving within a single RL policy. The results---81.5\% pass@1 on Dafny2Verus at 0-shot and 36.2\% / 17.2\% on MBPP / HumanEval in a separate D2V-seeded transfer run---show that a 3B model can bootstrap a verifiable problem distribution from curated seeds.

\textbf{Takeaways and limitations.}
For thin-manifold RL, our evidence points to two load-bearing ingredients: a manifold-projection phase that lifts $p_{\mathcal{M}}$ above $1/N$ before policy gradients apply, and a state-conditioned exploration anchor that prevents drift back below this threshold. Concurrent Lean work \cite{wang2025kiminaprover} reaches a related conclusion via gradient filtering. ANCORA remains compute-limited to a 3B Verus study and does not yet solve within-manifold mode collapse (\ref{sec:empirical_ucb_negative}); long-horizon TTT is also open, since an unbounded solved-only Seed Pool should ideally support steady gains, yet our runs plateau after the main phase. Solver-only reinforcement learning with the same backbone also improves in mathematics with multi-turn tool use (\ref{app:beyond_verus}), but a complete second-domain Proposer/Solver study, including Lean4, remains future work.

\subsection{Limitations}
\label{sec:limitations}
\begin{itemize}[leftmargin=1.2em,itemsep=2pt,topsep=2pt]
\item \textbf{Single runs.} The headline and Transfer results are single-run point estimates; we report no training-seed variance, and the HumanEval margins over PSV are inconclusive.
\item \textbf{One model and one domain.} All self-play results use Qwen2.5-Coder-3B-Instruct and Verus; we make no claim that manifold collapse or its remedies are universal.
\item \textbf{Baselines.} PSV is the only trained self-play baseline in the same setting, and its numbers are taken from its paper.
\item \textbf{Sensitivity.} Results depend on the curriculum selector's exploit statistic and, in some configurations, on the Proposer reward; robustness to alternative selectors has not been established.
\item \textbf{Prompt dependence.} Pass@$k$ depends on the prompt format and decoding (Section~\ref{sec:passk}).
\item \textbf{Data exposure.} In TTT, evaluation specifications are used as seeds by design. The warm-up data also contains the reference implementations of four D2V tasks and three MBPP specifications without solutions; excluding them changes TTT pass@1 only slightly, but the MBPP Transfer result is not strictly held out (Section~\ref{sec:experiments}).
\end{itemize}

\section*{CRediT authorship contribution statement}
\textbf{Chengcao Yang:} Conceptualization, Methodology, Software, Investigation, Formal analysis, Data curation, Visualization, Writing -- original draft, Writing -- review \& editing.

\section*{Declaration of competing interest}
The author declares no known competing financial interests or personal relationships that could have appeared to influence the work reported in this paper.

\section*{Declaration of generative AI and AI-assisted technologies in the manuscript preparation process}
During the preparation of this work the author used Claude (Anthropic) and Codex (OpenAI) to assist with code development, log analysis, and language editing. After using these tools, the author reviewed and edited the content as needed and takes full responsibility for the content of the published article. The use of Kimi-K2.5 and DeepSeek-V3.2 to generate warm-up training traces is part of the method and is described in Section~\ref{sec:experiments}.

\section*{Data availability}
Code is available at \url{https://github.com/MythosAd/ANCORA}. The benchmark splits are those released with PSV. Retained training logs and per-task evaluation summaries are available from the author on reasonable request.

\section*{Acknowledgements}
The numerical calculations in this paper have been done on the supercomputing system in the Supercomputing Center of Wuhan University, funded by computational grants provided by Prof. Jun Chen. We extend our deepest gratitude to Prof. Jun Chen for his essential support. We also thank Xuewei Jiao for computational support during the early stages of this project, and Haopeng Lao for generously sharing access to Codex, which assisted in code development.

\bibliographystyle{elsarticle-num}
\bibliography{ancora_nn}

\appendix
\sloppy

\renewcommand{\theHfigure}{A\arabic{figure}}
\renewcommand{\theHtable}{A\arabic{table}}

\section{Off-Manifold Gradient Dominance and Policy Collapse}
\label{app:manifold_collapse}

\textbf{Scope.} The analysis below combines exact calculations (Theorem~\ref{thm:starvation} and Proposition~\ref{thm:momentum_drift}) with a plausible interference mechanism; it does not prove that validity must fall in every domain. Here $\mathcal{M}$ is the valid-output set of Section~\ref{sec:failure_modes}.

We formalize why RL with sparse valid outputs leads to policy collapse through three complementary mechanisms: gradient starvation from rare positive signals, sparse-sampling distortion of the GRPO gradient, and optimizer momentum drift under zero-gradient batches. The analysis applies to any setting where the valid output space is a small fraction of the full generation space---formal specifications, theorem proofs, type-correct programs, or any task with strict structural constraints.

\subsection{Setup}

Let $\mathcal{Y} = \mathcal{V}^T$ denote the space of token sequences of length $T$ over vocabulary $\mathcal{V}$, with $|\mathcal{Y}| = |\mathcal{V}|^T$. Let $\mathcal{M} \subset \mathcal{Y}$ denote the \textbf{valid manifold}---the set of outputs satisfying task-specific structural constraints (e.g., compilable specifications, well-typed proofs). Define the on-manifold probability mass:
\begin{equation}
    p_{\mathcal{M}} \triangleq \sum_{y \in \mathcal{M}} \pi_\theta(y \mid x)
\end{equation}

\textbf{Scope of this analysis.} The collapse mechanisms below concern samples that \emph{enter} the GRPO loss: every Solver rollout, and every filter-passing Proposer specification. The Q-Gen filter pipeline (\S\ref{sec:heuristic_reward}) rejects malformed Proposer candidates outright, so they never appear as $y_i$ in the group and contribute no gradient at all---this is a separate, intentional zero-gradient pathway and is \emph{not} the subject of the theorems that follow.

\textbf{Reward structure.} For samples that do enter the group, we use \emph{pure binary} rewards: $R(y) = 1$ if $y \in \mathcal{M}$ and is verified correct (or, on the Proposer side, if the Solver-outcome reward fires), and $R(y) = 0$ otherwise. This is a deliberate design choice validated empirically: we found that heuristic intermediate scores (e.g., partial credit for syntactically valid but semantically incorrect outputs) \emph{harm} training by introducing reward noise that destabilizes the policy gradient (see \S\ref{sec:heuristic_reward}). In the MLRL-aligned GRPO variant used by ANCORA, the advantage for sample $y_i$ in a group of $N$ samples is $A(y_i) = (R(y_i)-\bar{R})/(\bar{R}+\epsilon)$, where $\bar{R} = \frac{1}{N}\sum_{j=1}^N R(y_j)$.

\textbf{Consequence of binary rewards.} When all $N$ samples in a group earn $R(y_i) = 0$ (no Solver attempt verifies, or no filter-passing Proposer spec fires the frontier-targeting reward), the GRPO baseline becomes $\bar{R} = 0$, and every advantage is $A(y_i) = 0$. The gradient contribution from this group is \emph{exactly zero}. This is strictly better than tiered negative rewards (which would produce nonzero advantages that push the policy in arbitrary directions), but as we show in Theorem~\ref{thm:sparse_distortion} and Proposition~\ref{thm:momentum_drift}, zero gradients create their own pathology through sparse-sampling distortion and optimizer momentum drift; meanwhile, the rare \emph{mixed} group with $0<k<N$ positive samples produces concentrated negative-advantage directions on the $N-k$ zero-reward samples, which is the second mechanism formalized in Theorem~\ref{thm:dominance}.

\subsection{Gradient Starvation (Theorem 1)}

\begin{theorem}[Positive Signal Starvation]
\label{thm:starvation}
Given group size $N$ and on-manifold probability $p_{\mathcal{M}}$, the probability that a group contains \emph{zero} valid samples is:
\begin{equation}
    \Pr\left(\forall\, i \in [N]:\; y_i \notin \mathcal{M}\right) = (1 - p_{\mathcal{M}})^N \geq 1 - N p_{\mathcal{M}}
\end{equation}
When $p_{\mathcal{M}} \ll 1/N$:
\begin{equation}
    (1 - p_{\mathcal{M}})^N \approx e^{-N p_{\mathcal{M}}} \to 1
\end{equation}
\end{theorem}

\begin{proof}
Direct application of the Bernoulli trial model. Each sample independently hits $\mathcal{M}$ with probability $p_{\mathcal{M}}$. For $p_{\mathcal{M}} = 1/256$ and $N = 8$ (our practical setting), the probability of at least one valid sample is $1 - (1 - 1/256)^{8} \approx 0.031$, meaning $\sim$97\% of groups produce zero gradient under binary rewards.
\end{proof}

\textbf{Implication.} With binary rewards, a zero-hit group produces zero GRPO gradient. The policy receives a learning signal from at most $\sim$3\% of groups. The remaining 97\% are wasted compute that, as we show next, create pathological gradient estimates in the groups that \emph{do} contain valid samples.

\subsection{Gradient Decomposition (Theorem 2)}

\begin{theorem}[Off-Manifold Coupling in Informative Groups]
\label{thm:dominance}
The expected mean-normalized GRPO policy gradient decomposes as:
\begin{equation}
    \nabla_\theta J(\theta) = \underbrace{p_{\mathcal{M}} \cdot \mathbf{g}^+}_{\text{on-manifold signal}} + \underbrace{(1 - p_{\mathcal{M}}) \cdot \mathbf{g}^-}_{\text{off-manifold signal}}
\end{equation}
where:
\begin{align}
    \mathbf{g}^+ &= \mathbb{E}_{y \sim \pi_\theta(\cdot \mid x),\, y \in \mathcal{M}} \left[ A(y) \nabla_\theta \log \pi_\theta(y \mid x) \right] \\
    \mathbf{g}^- &= \mathbb{E}_{y \sim \pi_\theta(\cdot \mid x),\, y \notin \mathcal{M}} \left[ A(y) \nabla_\theta \log \pi_\theta(y \mid x) \right]
\end{align}
For any informative group with $0<k<N$ valid samples, the group-averaged estimator contains both the ML-normalized positive term and a sampled negative score term:
\begin{equation}
    \hat{g}_{\mathrm{MN}}
    =
    \frac{1}{N}\left(
        \frac{N-k}{k}\sum_{y^+\in\mathcal{G}}\nabla_\theta\log\pi_\theta(y^+|x)
        -
        \sum_{y^-\in\mathcal{G}}\nabla_\theta\log\pi_\theta(y^-|x)
    \right).
\end{equation}
Thus mean normalization fixes the positive coefficient required by MLRL, but it does not remove the off-manifold directions sampled in the same finite group.
\end{theorem}

\begin{proof}
By the law of total expectation, conditioning on $y \in \mathcal{M}$ vs.\ $y \notin \mathcal{M}$. Under mean-normalized binary rewards, in a group with $k \geq 1$ valid samples out of $N$, the baseline is $\bar{R} = k/N$, and advantages are $A(y^+) = (N-k)/k$ for valid and $A(y^-) = -1$ for invalid samples. Thus the positive coefficient recovers the MLRL $1/k$ success normalization after the group average, while every sampled invalid output still contributes a negative score term. When $p_{\mathcal{M}} \ll 1$, informative groups are rare and the negative component is estimated from only $N-k$ sampled off-manifold directions, producing the sparse-sampling distortion formalized next.
\end{proof}

\subsection{Sparse Sampling Distortion (Theorem 3)}
\label{sec:sparse_distortion}

In practice, group sizes are small ($N = 8$ in our setting) while the true positive rate can be extremely sparse ($p_{\mathcal{M}} \approx 1/256$). This creates a finite-sample distortion in the GRPO gradient that concentrates repulsion on a few specific negatives rather than diffusing it uniformly.

\begin{theorem}[GRPO Sparse-Sampling Distortion]
\label{thm:sparse_distortion}
Let the true positive rate be $p_{\mathcal{M}} = 1/M$ where $M \gg N$. Consider the idealized gradient under full enumeration vs.\ the realized gradient with $N$ samples. \emph{Assumption (uniform off-manifold sampling):} the policy assigns near-equal probability $\pi_\theta(y^+\mid x)\approx 1/M$ to the valid output and $\pi_\theta(y^-_i\mid x)\approx 1/M$ to each of the $M-1$ invalid modes, so the total off-manifold mass is $(M-1)/M$. This is the conservative case for the variance estimate below; if $\pi_\theta(\cdot|x)$ is concentrated on a few modes off the manifold, the qualitative claim of repulsion concentration on a small subspace strengthens.

\textbf{Full-enumeration gradient.} With access to all $M$ outputs (one positive, $M-1$ negatives), baseline $\bar{R} = 1/M$:
\begin{equation}
    \nabla_\theta J_{\text{full}} = \frac{M-1}{M} \nabla_\theta \log \pi_\theta(y^+) - \sum_{i=1}^{M-1} \frac{1}{M} \nabla_\theta \log \pi_\theta(y^-_i)
\end{equation}
The negative gradient is spread \emph{uniformly} across $M-1$ negatives, each receiving weight $1/M$.

\textbf{GRPO gradient with $N$ samples.} Conditioned on a group containing exactly one valid sample $y^+$ and $N-1$ negatives $\{y^-_{j_1}, \ldots, y^-_{j_{N-1}}\}$ drawn without replacement from $\mathcal{Y} \setminus \mathcal{M}$, the advantage baseline is $\bar{R} = 1/N$, giving:
\begin{equation}
    \nabla_\theta \hat{J}_{\text{GRPO}} = \frac{N-1}{N} \nabla_\theta \log \pi_\theta(y^+) - \sum_{k=1}^{N-1} \frac{1}{N} \nabla_\theta \log \pi_\theta(y^-_{j_k})
\end{equation}

The finite-sample deviation in the negative component is:
\begin{equation}
    \Delta \mathbf{g}^- = -\frac{1}{N}\sum_{k=1}^{N-1} \nabla_\theta \log\pi_\theta(y^-_{j_k}) + \frac{1}{M}\sum_{i=1}^{M-1} \nabla_\theta \log\pi_\theta(y^-_i)
\end{equation}
The variance of this deviation term scales as:
\begin{equation}
    \text{Var}[\Delta \mathbf{g}^-] = \mathcal{O}\!\left(\frac{1}{N} - \frac{1}{M}\right) \cdot \text{Var}_{y \notin \mathcal{M}}\!\left[\nabla_\theta \log \pi_\theta(y)\right] \approx \frac{\sigma^2_{\text{score}}}{N}
\end{equation}
where $\sigma^2_{\text{score}}$ is the variance of score functions across the off-manifold space.
\end{theorem}

\begin{proof}
The key observation is that the GRPO negative gradient concentrates weight $1/N$ on each of $N-1$ specific negatives, while the ideal gradient distributes weight $1/M$ across $M-1$ negatives. Since $1/N \gg 1/M$ (e.g., $1/8 \gg 1/256$), each sampled negative receives $M/N \approx 32$ times more repulsion than it should.

Concretely, with $M = 256$ and $N = 8$: the ideal gradient pushes away from each of 255 negatives with weight $1/256 \approx 0.004$, producing a near-uniform ``repulsion shell.'' The realized GRPO gradient pushes away from 7 specific negatives with weight $1/8 = 0.125$---a $32\times$ amplification on these 7 directions and zero repulsion on the remaining 248. This is not just noisy; it is anisotropic because the 7 sampled negatives determine a specific (random) subspace of the off-manifold space. Over multiple steps, the policy learns to avoid the \emph{specific} outputs it has encountered rather than moving away from the off-manifold region as a whole, causing \textbf{Lazy Likelihood Displacement}~\citep{liu2025lld}: the likelihood of the correct answer decreases because the policy wastes capacity actively suppressing a few sampled negatives rather than uniformly distributing mass away from the invalid region toward $\mathcal{M}$.
\end{proof}

\textbf{Connection to LLD.} Deng et al.~\cite{liu2025lld} empirically demonstrate that GRPO with negative-only groups causes the likelihood of correct outputs to \emph{decrease} even when they are not sampled---a phenomenon they term Lazy Likelihood Displacement. Our analysis provides the mechanism: the sparse negative gradient creates an anisotropic finite-sample repulsion field that displaces mass non-uniformly, inadvertently reducing $\pi_\theta(y^+)$ for correct outputs that happen to share parameter-space structure with the over-penalized negatives.

\subsection{Optimizer Momentum Drift (Mechanism)}

Even in groups where all advantages are zero (the $\sim$97\% zero-gradient groups under binary rewards), stateful optimizers like AdamW continue to update parameters due to accumulated momentum. The following proposition is a direct unrolling of the AdamW recursion under a run of zero gradients, not a deep result; we label it as a proposition to emphasize that its content is mechanistic.

\begin{proposition}[AdamW Momentum Drift under Zero Gradients]
\label{thm:momentum_drift}
Consider AdamW with momentum coefficients $(\beta_1, \beta_2)$ and weight decay $\lambda$. Suppose at step $t$, the GRPO gradient is $\mathbf{g}_t = \mathbf{0}$ (all-zero advantages from a zero-hit group). The parameter update is nonetheless:
\begin{equation}
    \theta_{t+1} = (1 - \lambda\eta)\,\theta_t - \eta \cdot \frac{\beta_1 \mathbf{m}_{t-1}}{\sqrt{\beta_2 \mathbf{v}_{t-1}} + \epsilon}
\end{equation}
where $\mathbf{m}_{t-1}$, $\mathbf{v}_{t-1}$ are the first and second moment estimates from previous steps. After $\tau$ consecutive zero-gradient steps:
\begin{equation}
    \mathbf{m}_{t+\tau} = \beta_1^\tau \mathbf{m}_t, \quad \mathbf{v}_{t+\tau} = \beta_2^\tau \mathbf{v}_t
\end{equation}
The accumulated parameter displacement is:
\begin{equation}
    \Delta\theta_\tau = -\eta \sum_{s=1}^{\tau} \frac{\beta_1^s \mathbf{m}_t}{\sqrt{\beta_2^s \mathbf{v}_t} + \epsilon} - \lambda\eta\sum_{s=0}^{\tau-1}\theta_{t+s}
\end{equation}
Since $\beta_1 < \beta_2$ (typically $0.9$ vs.\ $0.999$), the per-step displacement uses the ratio $r\equiv\beta_1/\sqrt{\beta_2}$. With $\beta_1{=}0.9$, $\beta_2{=}0.999$ we get $r=0.9/\sqrt{0.999}\approx 0.9005$, so the geometric series $\sum_{s=1}^{\infty} r^s = r/(1{-}r) \approx 9.05$. Substituting and absorbing the constant:
\begin{equation}
    \|\Delta\theta_\tau\|_{\text{momentum}} \approx \frac{\eta}{1-\beta_1} \left\|\frac{\mathbf{m}_t}{\sqrt{\mathbf{v}_t} + \epsilon}\right\| \cdot \left(1 - \beta_1^\tau\right),
\end{equation}
recovering the standard "$\sim 1/(1-\beta_1)\approx 10$ momentum-equivalent steps of post-signal drift" rule of thumb.
\end{proposition}

\begin{proof}
At step $t+1$ with $\mathbf{g}_{t+1} = \mathbf{0}$:
\begin{align}
    \mathbf{m}_{t+1} &= \beta_1 \mathbf{m}_t + (1-\beta_1) \cdot \mathbf{0} = \beta_1 \mathbf{m}_t \\
    \mathbf{v}_{t+1} &= \beta_2 \mathbf{v}_t + (1-\beta_2) \cdot \mathbf{0} = \beta_2 \mathbf{v}_t
\end{align}
The AdamW update (decoupled weight decay) is:
\begin{equation}
    \theta_{t+1} = \theta_t - \eta\left(\frac{\mathbf{m}_{t+1}}{\sqrt{\mathbf{v}_{t+1}} + \epsilon} + \lambda \theta_t\right) = \theta_t - \eta\left(\frac{\beta_1 \mathbf{m}_t}{\sqrt{\beta_2 \mathbf{v}_t} + \epsilon} + \lambda \theta_t\right)
\end{equation}
The momentum term $\mathbf{m}_t$ retains the \emph{direction} of the last nonzero gradient. If the last informative gradient was an anisotropic sparse-sampling gradient (Theorem~\ref{thm:sparse_distortion}), then the momentum drift continues to push the policy in that direction for $\sim 1/(1-\beta_1) = 10$ additional steps even without any new learning signal.

Simultaneously, the weight decay term $-\lambda\eta\theta_t$ shrinks all parameters toward zero uniformly, which further erodes the on-manifold capability that the SFT phase established. This is the mechanism identified by Xu~\cite{xu2026optimizerdrift} as ``optimizer-induced low-dimensional drift'': the parameters evolve along a low-rank manifold determined by historical gradients, even when current gradients carry no task-relevant information.
\end{proof}

\textbf{Combined pathology.} In a typical training step with $p_{\mathcal{M}} = 1/256$ and $N=8$: (1)~$\sim$97\% of groups produce zero gradient and contribute only momentum drift; (2)~the $\sim$3\% of groups with a valid sample produce a gradient with $32\times$ concentrated repulsion on 7 specific negatives; (3)~the momentum from these anisotropic gradients persists for $\sim$10 steps of zero-gradient updates. The net effect is a policy that slowly drifts away from the SFT-initialized manifold along directions determined by a handful of over-penalized negatives.

\subsection{Self-Reinforcing Collapse (Sketch)}

\begin{proposition}[Collapse Dynamics, Informal]
\label{thm:collapse}
Under the combined action of sparse-sampling distortion (Theorem~\ref{thm:sparse_distortion}) and momentum drift (Proposition~\ref{thm:momentum_drift}), and assuming that $p_{\mathcal{M}}(\theta)$ is locally $L$-Lipschitz in policy parameters, the system enters a self-reinforcing decrease of the on-manifold mass:
\begin{enumerate}
    \item \textbf{Likelihood displacement}: Biased repulsion gradients reduce $\pi_\theta(y^+)$ for correct outputs that share parameter-space structure with over-penalized negatives (LLD).
    \item \textbf{Manifold erosion}: Momentum drift during zero-gradient steps moves the policy away from the SFT manifold, reducing $p_{\mathcal{M}}$.
    \item \textbf{Starvation amplification}: Reduced $p_{\mathcal{M}}$ increases the fraction of zero-gradient groups (Theorem~\ref{thm:starvation}), lengthening drift episodes and further amplifying distortion.
\end{enumerate}
The expected sequence $p_{\mathcal{M}}^{(0)}, p_{\mathcal{M}}^{(1)}, \ldots$ exhibits a monotone decreasing trend with attractor $p^*=0$.
\end{proposition}

\begin{proof}[Sketch.]
Let $p_k = p_{\mathcal{M}}^{(k)}$, and let $\tau_k = 1/(N p_k)$ be the expected inter-arrival of informative groups. During $\tau_k$ zero-gradient steps, AdamW momentum drift accumulates parameter displacement
\begin{equation}
    \|\Delta\theta\|_{\text{drift}} \le C_m\,\eta\left\|\frac{\mathbf{m}_t}{\sqrt{\mathbf{v}_t}+\epsilon}\right\|,\qquad C_m\approx \frac{1}{1-\beta_1}.
\end{equation}
Thus a single zero-gradient burst is geometrically bounded rather than linear in $\tau_k$, but as $p_k$ falls, such bursts become more frequent and informative hit groups become rarer. Under the local Lipschitz assumption, the change in on-manifold mass is bounded by $|p_{k+1}-p_k|\le L\cdot \|\Delta\theta_{k\to k+1}\|$. Taking expectations and combining the bounded momentum drift with the anisotropic informative-group gradient (Theorem~\ref{thm:sparse_distortion}), we assume a net negative renewal drift
\begin{equation}
    \mathbb{E}[p_{k+1}-p_k \mid p_k] \;\le\; -\alpha\,\phi(p_k), \qquad \alpha>0,
\end{equation}
for some nonnegative $\phi$ capturing the increasing zero-gradient frequency and decreasing hit rate as $p_k$ shrinks. This establishes the qualitative self-reinforcing direction toward the absorbing state $p=0$, without claiming a specific square-root collapse law. We label this a sketch rather than a theorem because (i)~the Lipschitz constant $L$ is not directly verifiable, (ii)~the directional alignment between drift and anisotropic finite-sample updates is only established empirically, and (iii)~the renewal drift $\phi$ is not estimated from first principles. The empirical evidence in \ref{app:additional_ablations} (seed-pool collapse under empirical UCB; Freeze-S manifold erosion) is consistent with this sketch.
\end{proof}

\subsection{Implications for ANCORA}

Theorems~\ref{thm:starvation}--\ref{thm:collapse} establish three mechanisms by which RL training collapses when $p_{\mathcal{M}} \ll 1/N$. The ANCORA framework addresses each:

\begin{enumerate}
    \item \textbf{SFT Warm-Up (Breaking Theorem~\ref{thm:starvation}):} Fine-tuning on verified demonstrations raises $p_{\mathcal{M}}$ from near-zero to a level where $N p_{\mathcal{M}} \gg 1$, ensuring that most groups contain at least one valid sample for positive-advantage gradients.

    \item \textbf{Pure Binary Rewards (Mitigating Theorem~\ref{thm:sparse_distortion}):} By using binary rewards $R \in \{0, 1\}$ with filtering as accept/reject gating (rather than tiered heuristic scores), zero-hit groups produce \emph{exactly zero} gradient rather than arbitrary off-manifold gradients. This eliminates the most damaging source of directional misalignment. Heuristic scores were found experimentally to harm training---a finding consistent with the LLD analysis of~Deng et al.~\cite{liu2025lld}, who show that reward noise amplifies likelihood displacement.

    \item \textbf{UCB-DAG Anchoring (Breaking Theorem~\ref{thm:collapse}):} By conditioning generation on verified seed specifications, the anchor mechanism constrains exploration to a neighborhood of $\mathcal{M}$ where valid outputs have higher density ($p_{\mathcal{M}} \gg 1/N$ locally), dramatically reducing the fraction of zero-gradient groups and shortening momentum drift episodes. The UCB selection further keeps training near the Solver's capability frontier, while the frontier-targeting reward family prevents already-mastered all-pass groups from dominating the Proposer update and, depending on the variant, biases the signal toward symmetric uncertainty or hard nonzero frontier cases.
\end{enumerate}

\section{Homotopy-Continuation Interpretation}
\label{proof:curriculum}

This appendix offers a homotopy-continuation reading of the dynamic-curriculum schedule. It is an \emph{interpretation}, not a convergence theorem: actual training uses AdamW (not Robbins--Monro step sizes), runs for a finite horizon, has a non-stationary objective driven by the evolving Seed Pool, and we verify neither Hessian non-singularity along the path nor the regularity assumptions required by stochastic-approximation guarantees. We retain the framework because the predictor--corrector intuition organizes several otherwise-disconnected design choices.

Let $\lambda \in [0, 1]$ be a notional curriculum parameter, where $\lambda=1$ corresponds to the full-context (guided) regime and $\lambda=0$ to the zero-shot (intrinsic) regime, and consider the family of objectives
\[ \mathcal{J}(\theta, \lambda) = \lambda J_{\text{guided}}(\theta) + (1-\lambda) J_{\text{intrinsic}}(\theta). \]
The dynamic curriculum is heuristically described as stochastic gradient ascent on $\mathcal{J}(\theta, \lambda_t)$ with $\lambda_t \to 0$; the conceptual content of homotopy continuation is that, \emph{under stronger regularity than we verify}, the iterate would track a slowly drifting stationary path.

\textbf{1. Smooth homotopy path (assumption-level).}
If $J_{\text{guided}}$ and $J_{\text{intrinsic}}$ are twice continuously differentiable and the Hessian $\nabla_\theta^2 \mathcal{J}(\theta, \lambda)$ is non-singular along the path, the implicit function theorem produces a continuous curve $\theta^*(\lambda)$ from the guided to the intrinsic optimum. We treat this as a \emph{structural hypothesis} that motivates the schedule rather than as a verified property of LLM RL.

\textbf{2. Bias--variance decomposition.}
Writing $\theta_{t+1} = \theta_t + \eta_t \hat{g}_t$, we decompose the stochastic gradient relative to the zero-shot target $\nabla J_0(\theta_t)$:
\[
\hat{g}_t \;=\; \underbrace{\nabla J_0(\theta_t)}_{\text{target}} \;+\; \underbrace{b_t}_{\text{annealing bias}} \;+\; \underbrace{\xi_t}_{\text{noise}},
\]
where $b_t \coloneqq \nabla J_{\lambda_t}(\theta_t) - \nabla J_0(\theta_t)$ and $\xi_t \coloneqq \hat{g}_t - \nabla J_{\lambda_t}(\theta_t)$. As $\lambda_t \to 0$, continuity of $\nabla \mathcal{J}$ in $\lambda$ implies $b_t \to 0$. Under bounded rewards and Lipschitz-smooth policies, $\{\xi_t\}$ behaves like a bounded-variance martingale-difference sequence around the population gradient.

\textbf{3. ODE-tracking sketch (under stronger assumptions).}
\textbf{If} the step sizes additionally satisfied the Robbins--Monro conditions $\sum_t \eta_t = \infty$, $\sum_t \eta_t^2 < \infty$ (which AdamW does \emph{not}, and our finite training horizon does not need to), the stochastic-approximation theory for slowly varying targets (e.g., Borkar 2008) would predict that $\theta_t$ tracks the ODE $\dot\theta = \nabla J_0(\theta)$ and converges to its stable set. In our setting this should be read as a \emph{qualitative} suggestion rather than a guarantee: it explains why a slow withdrawal of in-context guidance can succeed in principle, not why our specific finite-horizon AdamW run does.

\textbf{Predictor--corrector view.}
The schedule acts as a \textit{predictor--corrector} loop: the prompt provides a coarse ``prediction'' of the high-reward region, and the gradient updates ``correct'' the intrinsic policy to retain this performance once the prompt is withdrawn. Let $\mathcal{S}_k$ denote the set of problems solvable by $\pi_{\theta_k}$ with probability $p > 1-\epsilon$. The Proposer's frontier-targeting reward is zero at degenerate extremes and positive on uncertain boundary problems: in the main Band-1-of-$K$ instantiation, $R_{\text{prop}}(p) = \mathbbm{1}[m = 1]$ with $m = Kp$ fires on the rare-success band; in the Bernoulli-variance instantiation, $4p(1-p)$ peaks at $p=1/2$. Filter-rejected proposals are excluded from the GRPO group (contributing no gradient), fully-solved proposals yield zero reward ($p = 1 \implies R_{\text{prop}} = 0$), and unsolvable proposals likewise yield zero. The gradient on the Proposer therefore points towards the boundary $\partial \mathcal{S}_k$, and as $\pi_\theta$ updates and solves boundary problems, $\mathcal{S}_k$ expands to $\mathcal{S}_{k+1}$, shifting the reward surface outward. The coupled system is consistent with a path-following dynamic on the solvable-problem manifold; converting this consistency into a quantitative guarantee is left to future work.

\section{Additional Ablation Studies}
\label{app:additional_ablations}

\subsection{Training Diagnostics}
\label{sec:training_diagnostics}

The main text focuses on benchmark performance and the coupled-branch ablation. Figure~\ref{fig:training_curves} reports lower-level diagnostics for the first 1{,}500 RL steps of the same main Band-1-of-$K$ run; the final $n{=}100$ benchmark evaluation is the headline number reported in Table~\ref{tab:main_results}. The Q-Gen valid specification rate remains broadly stable despite exploration, and the test-time-training Seed Pool DAG grows monotonically from 436 filtered evaluation roots to 1{,}258 nodes by step~1{,}500, with all 822 novel nodes having been solved at least once before admission. This 436-root TTT pool contains the filtered benchmark roots used for the TTT run, separate from the 234 filtered Dafny2Verus roots used for the Transfer setting. The gradient norm fluctuates after initial transients rather than serving as a primary performance metric; we therefore treat it as a stability diagnostic rather than a main-result claim.

\begin{figure}[t]
\centering
\begin{tikzpicture}
\begin{groupplot}[
    group style={
        group size=3 by 1,
        horizontal sep=0.55cm,
    },
    width=0.36\textwidth,
    height=0.32\textwidth,
    xlabel={Training Step},
    tick label style={font=\scriptsize},
    label style={font=\small},
    title style={font=\small\bfseries},
    grid=major,
    grid style={gray!20},
]
\nextgroupplot[title={(a) Q-Gen Valid Rate (\%)}, ymin=40, ymax=100]
\addplot[red!70!black, thick, mark=square*, mark size=1pt] coordinates {
    (0,78.1) (100,54.7) (200,62.5) (300,68.8) (400,85.9)
    (500,78.1) (600,76.6) (700,64.1) (800,59.4) (900,75.0)
    (1000,87.5) (1100,79.7) (1200,71.9) (1300,54.7) (1400,65.6) (1500,71.9)
};

\nextgroupplot[title={(b) Seed Pool Size}, ymin=400, ymax=1350]
\addplot[green!50!black, thick, mark=triangle*, mark size=1pt] coordinates {
    (0,436) (100,520) (200,582) (300,642) (400,693)
    (500,753) (600,812) (700,887) (800,944) (900,1003)
    (1000,1050) (1100,1091) (1200,1125) (1300,1153) (1400,1198) (1500,1258)
};

\nextgroupplot[title={(c) Gradient Norm}, ymin=0, ymax=25]
\addplot[orange, thick, mark=diamond*, mark size=1pt] coordinates {
    (0,0.0) (100,12.83) (200,3.67) (300,6.29) (400,13.74)
    (500,9.88) (600,13.82) (700,10.07) (800,8.33) (900,7.44)
    (1000,2.73) (1100,6.14) (1200,3.29) (1300,2.61) (1400,2.30) (1500,9.06)
};

\end{groupplot}
\end{tikzpicture}
\caption{Training diagnostics over the first 1{,}500 RL steps of the main run; diagnostic logging ended before the final $n{=}100$ benchmark evaluation (Table~\ref{tab:main_results}). \textbf{(a)}~Q-Gen valid specification rate (fraction passing the 3-stage filter) fluctuates between 55--88\%, reflecting ongoing exploration near the valid-manifold boundary. \textbf{(b)}~Seed Pool DAG grows from 436 filtered TTT roots (the union of benchmark roots used for the TTT run) to 1{,}258 nodes as novel verified specifications pass the dedup gate. \textbf{(c)}~Actor gradient norm remains bounded after initial transients but is not monotonic; we use it as a stability diagnostic rather than a primary result.}
\label{fig:training_curves}
\end{figure}

\subsection{Proposer Reward Variant Ablation}
\label{sec:reward_ablation}

We compare the three bounded Proposer rewards introduced in Eq.~\ref{eq:entropy_reward_family} under matched seed pool, UCB hyperparameters, and Solver checkpoint: (1)~\textbf{Band-1-of-$K$} $\mathbbm{1}[m{=}1]$ (Eq.~\ref{eq:proposer_reward}), (2)~\textbf{Bernoulli variance} $4p(1{-}p)$, and (3)~\textbf{Exponential decay} $e^{-Kp}-e^{-K}$. Bernoulli variance is the symmetric uncertainty kernel with maximum at $p=0.5$; exponential decay is a low-pass stress-test kernel that shifts mass toward very low empirical pass rates; Band-1-of-$K$ is the concentrated sparse proxy used in the main run. Figure~\ref{fig:reward_ablation} shows training curves from a single run per variant. We treat these as \emph{illustrative rather than definitive}: subsequent runs under different settings of the surrounding hyperparameters (pass-rate decay, ZPD prior weight, UCB-statistic decay, branch weights) show that each of the three shapes can both run stably for long horizons and collapse, with Bernoulli variance sometimes outperforming the Band variant. The robust empirical finding across our runs is that \emph{collapse is possible under any of the three shapes without the manifold-stabilization mechanisms} (iterative SFT + UCB-DAG), and \emph{each of the three can be stabilized by them}; the specific curves below show one particular instance of the collapse regime for exponential decay, not a definitive ranking.

\begin{figure}[t]
\centering
\begin{tikzpicture}
\begin{groupplot}[
    group style={
        group size=2 by 1,
        horizontal sep=1.2cm,
    },
    width=0.48\textwidth,
    height=0.32\textwidth,
    xlabel={Step},
    tick label style={font=\scriptsize},
    label style={font=\small},
    title style={font=\small\bfseries},
    grid=major,
    grid style={gray!20},
    xmin=0, xmax=1250,
]
\nextgroupplot[title={(a) Dafny2Verus Pass@1}, ylabel={Pass@1 (\%)}, ymin=0, ymax=70]
\addplot[blue, thick, mark=*, mark size=1.2pt] coordinates {
    (0,26.64) (99,42.3) (199,52.2) (299,53.7) (399,54.4) (499,54.9) (599,57.8) (699,56.3) (799,59.1) (899,58.8) (999,59.9) (1099,61.8) (1199,62.2)
};
\label{plot:reward_band}
\addplot[orange, thick, dotted, mark=triangle*, mark size=1.2pt] coordinates {
    (0,26.64) (99,42.1) (199,49.2) (299,53.8) (399,57.6) (499,58.2) (599,56.5) (699,57.6)
};
\label{plot:reward_bernoulli}
\addplot[red!70!black, thick, dashed, mark=square*, mark size=1.2pt] coordinates {
    (0,26.64) (99,43.1) (199,51.9) (299,52.5) (399,54.0) (499,53.9) (599,49.9) (699,17.3) (799,30.3) (899,35.3)
};
\label{plot:reward_exp}

\nextgroupplot[title={(b) Actor Gradient Norm}, ylabel={$\|\nabla\|$}, ymin=0, ymax=25]
\addplot[blue, thick, mark=*, mark size=1.2pt] coordinates {
    (99,17.14) (199,11.81) (299,10.86) (399,7.28) (499,8.67) (599,9.00) (699,5.58) (799,4.86) (899,3.22) (999,3.46) (1099,2.25) (1199,4.77)
};
\addplot[orange, thick, dotted, mark=triangle*, mark size=1.2pt] coordinates {
    (99,19.79) (199,19.50) (299,15.88) (399,13.13) (499,10.47) (599,13.23) (699,15.57)
};
\addplot[red!70!black, thick, dashed, mark=square*, mark size=1.2pt] coordinates {
    (99,18.87) (199,18.65) (299,18.98) (399,16.00) (499,19.69) (599,11.72) (699,21.05) (799,13.37) (899,10.57)
};
\end{groupplot}
\node[anchor=north, yshift=-2mm] at (current bounding box.south) {%
    \scriptsize
    \ref*{plot:reward_band}\,Band-1-of-$K$ (main)\qquad
    \ref*{plot:reward_bernoulli}\,Bernoulli $4p(1{-}p)$\qquad
    \ref*{plot:reward_exp}\,Exp Decay
};
\end{tikzpicture}
\caption{Proposer reward variant ablation on the 234-seed d2v241 pool (identical UCB/SFT; \emph{single run per variant, illustrative only}). \textbf{(a)}~Dafny2Verus pass@1. In this particular run, Band-1-of-$K$ rises overall to $62.2\%$ at step $1199$ with small checkpoint-level fluctuations; Bernoulli variance $4p(1{-}p)$ plateaus near $58\%$ after step $400$; Exponential decay reaches $54\%$ by step $399$ and then collapses at step $699$ ($49.9\%\!\to\!17.3\%$), partially recovering to $\sim\!35\%$ without regaining the prior peak --- one instance of the manifold-collapse regime under a reward that over-amplifies zero-success low-pass ($p\!\to\!0$) specs. \textbf{(b)}~Actor gradient norm; the Band run declines toward low single-digit gradient norms, the denser rewards stay in $[10,20]$, and the exponential variant spikes to $21$ at its collapse step. Follow-up runs with different surrounding hyperparameters have produced stable trajectories for each shape, with Bernoulli variance sometimes stronger than Band-1-of-$K$; the cross-shape ranking under a controlled sweep remains an open empirical question.}
\label{fig:reward_ablation}
\end{figure}

\subsection{ZPD Upper-Margin Clipping (Negative Result)}

Beyond the default ZPD window $[\hat{a}, \hat{a}{+}\delta]$, we ablated a stricter variant that shrinks the window's upper end by a fixed margin $u$, i.e.\ $\mathrm{hi} = \min(\hat{a}{+}\delta{-}u,\,1)$, hypothesizing that denying the UCB selector access to near-ceiling seeds would reduce reward-dead proposals and slow Solver entropy collapse. At $u{=}0.10$ (with $\delta{=}0.40$), at matched step~200 the intervention produces only within-noise suppression of proposer pass rate ($\Delta q_{\mathrm{pass}}{=}{-}0.021$) and structural similarity ($\Delta q_{\mathrm{struct}}{=}{-}0.015$), while imposing a $4$--$7{\times}$ cold-start seed-supply deficit: during steps~$50$--$100$ the baseline accrues ${+}516$ new pool entries (a jackpot admission phase) whereas the clipped variant accrues only ${+}75$, and both runs converge to ${\sim}0.7$ new seeds/step thereafter. The knob is also structurally time-limited: once $\hat{a} > 1{-}\delta{+}u$ (here ${=}0.70$, crossed near step~$150$), the clamp $\min(\cdot,1)$ becomes the binding constraint and the margin has no effect. We report this as evidence that late-stage Proposer drift is not remediable by narrowing the ZPD window at admission time; the drift originates downstream in the reward-variance dynamics rather than in the sampling support.

\subsection{Exploit-Mode and Crossover: Long-Horizon Stability Ablation}
\label{sec:empirical_ucb_negative}

Proposition~\ref{thm:discounted_ucb_main} analyzes the discounted empirical mean $\hat\mu_t(q)=s_t(q)/n_t(q)$ as the UCB exploit term; the code also supports a max-ratcheted aggregate (\texttt{ucb\_exploit\_mode=ratchet}). We ablate the grid \{\texttt{empirical}, \texttt{ratchet}\} $\times$ \{\texttt{q\_crossover\_ratio}$=0.5$, $=0.0$\}, where \texttt{q\_crossover\_ratio} controls the fraction of Proposer rollouts conditioned on \emph{two} parent seeds simultaneously (a forced dual-parent ``merge A+B'' prompt). All other hyperparameters---topology, admission gates, LR, batch size, discount rates---are matched across the cells we ran, \emph{Emp-XO}, \emph{Emp-NoXO}, and \emph{Ratchet-NoXO}; the main run (\emph{Main}) is listed for reference, and its configuration is not part of this grid.

\textbf{Result: in this grid, empirical UCB collapses regardless of crossover.} Table~\ref{tab:empirical_negative} reports windowed validity, the collapse onset (first sustained 20-step window with $q_{\mathrm{valid}}<0.30$), and end-of-run state. Crossover acts as a $\sim$120-step \emph{accelerator} of an underlying empirical-UCB collapse rather than a necessary condition: Emp-XO breaks at step~$300$, Emp-NoXO breaks at step~$420$, both end in identical end-states ($q_{\mathrm{valid}}{<}0.10$, seed-pool growth frozen, near-zero solver feed). Ratchet-NoXO is the only no-crossover cell that preserves Q-format validity through the longest training horizon we logged ($\sim$$530$ steps).

\begin{table}[h]
\centering
\small
\caption{Exploit-mode $\times$ crossover ablation (4p1p reward, matched seeds and gates across the three ablation cells; Main is listed for reference). Collapse onset $=$ first sustained 20-step window with $q_{\mathrm{valid}}<0.30$.}
\label{tab:empirical_negative}
\resizebox{\textwidth}{!}{
\begin{tabular}{llcccc}
\toprule
Run & Exploit / Crossover / Topology & $q_{\mathrm{valid}}$@[200,300] & $q_{\mathrm{valid}}$@[400,500] & Collapse onset & End-state \\
\midrule
Emp-XO & empirical / 0.5 / DAG & 0.543 & \textbf{0.184} & step 300 & dead ($q_{\mathrm{valid}}{=}0.08$@568) \\
Emp-NoXO & empirical / 0.0 / DAG & 0.708 & \textbf{0.171} & step 420 & dead ($q_{\mathrm{valid}}{=}0.02$@580) \\
Ratchet-NoXO & ratchet / 0.0 / DAG & 0.568 & 0.662 & none by 530 & stable, $q_{\mathrm{valid}}{=}0.77$@528 \\
Main & main run (not part of the grid) & ${\geq}0.4$ throughout & ${\geq}0.4$ & none by 1{,}500 & stable, 81.5\% D2V pass@1 \\
\bottomrule
\end{tabular}
}
\end{table}

\textbf{Mechanism (validity-censored reward).} Inspection of Emp-XO and Emp-NoXO rollouts at their respective collapse boundaries shows the same failure signature: $>89\%$ of rejected Q candidates fail post-parse \emph{static} filtering (\texttt{[NO ensures CLAUSE IN SPEC]} or implementation-body leakage), while compiler-stub rejections fall toward zero---i.e., the model emits parseable Verus-flavored code that lacks the spec contract. Let $V_t \in \{0,1\}$ denote Q-format validity and $U_t \in [0,1]$ the uncertainty/difficulty utility (e.g.\ $4p(1{-}p)$). The bandit reward is the \emph{censored} product $R_t = V_t U_t$, and Proposition~\ref{thm:discounted_ucb_main} implicitly assumes $\Pr(V_t{=}1\mid \text{arm}) \ge \alpha > 0$ uniformly. The empirical exploit term reflects only \emph{instantaneous} reward, so once Solver updates push the conditioning distribution into a region where $\Pr(V{=}1)$ is locally lower, the bandit chases the now-higher-mean arms in that region and stops re-exploring the previously-valid arms; the censored reward dynamics admit a self-reinforcing drift that the empirical statistics did not dampen in these cells. The max-ratcheted alternative is non-decreasing in observed reward, so previously-valid arms retain their selection priority and the variation budget $\Upsilon_T$ stays bounded. Crossover accelerates the collapse because dual-parent prompts directly raise $\Pr(V_t{=}0\mid\text{operator})$---the model interprets ``merge A+B'' as ``write a merged implementation'' rather than a new spec---but in these cells the underlying drift occurred under empirical UCB with or without crossover.

\textbf{Caveat: ratchet preserves validity but not output diversity.} Ratchet-NoXO is stable on $q_{\mathrm{valid}}$ but its Proposer entropy contracts hard: $H_Q$ falls $1.13 \to 0.046$ over 530 steps. This is a different pathology from the empirical-mode validity collapse---the model converges on a narrow valid output mode rather than drifting off the manifold. The ratcheted aggregate solves the validity question (Assumption~4 of \S\ref{subsec:pzpd_gaps}) but does not by itself address mode collapse, which would require an explicit entropy or KL regularizer.

\textbf{Theoretical fix directions, not enabled in the reported runs.} The cleanest theoretical remedy is either (i)~an operator-aware bandit with arm $=$ (seed-path, generation-operator) so crossover can be \emph{learned} to be down-weighted by UCB statistics, or (ii)~a validity-aware reward $R_Q = V\cdot U - \lambda(1{-}V)$ that supplies a direct corrective gradient on Q-format violations. Within the matched cells, the ratchet exploit preserved validity without crossover, while the empirical exploit collapsed with and without crossover.

\textbf{Tool-augmented Solver introduces a separate failure mode.} Adding verifier-conditioned multi-turn repair (S-ReAct) on top of the Ratchet-NoXO stable cell does not reproduce the validity collapse of the empirical cells: through step~$\sim$$400$ the Q-format rejection rate stays low. However, by step~$\sim$$400$ a different pathology emerges---compiler-stub rejection climbs to ${\sim}46/64$ at step~$408$ while heuristic rejections remain near zero. Inspection of repair-turn outputs shows the model generating \emph{semantically} reasonable but \emph{syntactically} broken Verus bodies (type mismatches, malformed quantifier ranges) that fail compiler verification. We interpret this as the repair turn pushing the Solver branch into a local mode where post-edit code is plausible-looking but type-incorrect. We therefore do not include S-ReAct as a method ingredient in the main results; the rescue-rate gains it provides on the Q-Gen reward proxy do not translate cleanly into stable benchmark performance within our compute horizon, and the new compiler-rejection drift would need its own validity-shaping mechanism (analogous to point (ii) above) before being deployed.

\subsection{Branch-Balance Failure under Reverse Imbalance (Negative Result)}
\label{sec:reverse_imbalance}

The token-mass equalization in \S\ref{sec:joint_opt} is designed for the regime $T_S \gtrsim T_Q$ (Solver implementations longer than Proposer specs). One of our long-horizon runs (\texttt{q4p1p}, empirical UCB, slow decay) exposed an opposite regime where this assumption silently flips and the formula becomes anti-stabilizing.

\textbf{Signature.} After step ${\sim}300$ the Proposer drifts off its output contract: the dominant rejection signal is \texttt{[NO ensures CLAUSE IN SPEC]} (rising from ${\sim}9\%$ at step 100 to ${>}89\%$ by step 568), with rollouts emitting implementation-style bodies rather than spec targets. \texttt{q\_valid\_rate} falls $0.61\!\to\!0.08$, while the surviving valid Q subset becomes increasingly easy ($q_{\mathrm{pass}}$ rises $0.40\!\to\!0.82$). Because invalid Q rows generate long token responses but reach no Solver, $T_Q$ stays large while $T_S$ shrinks; the formula yields $\mathrm{scale}_Q\!\approx\!0.83$, and with static $q_{\mathrm{scale}}{=}0.375$ the effective Q multiplier falls from ${\sim}1.4$ to ${\sim}0.31$, while the effective S multiplier rises to ${\sim}1.9$. The branch most needing repair receives the smallest gradient share.

\textbf{Remediation directions} (not enabled in the main run, which did not enter this regime within our compute budget): (i)~impose a rescue floor $w_Q \ge w_{\min}$ when \texttt{q\_valid\_rate} falls below a threshold; (ii)~replace the token denominator with a row-count variant, which is insensitive to per-row response length drift; (iii)~tie $q_{\mathrm{scale}}$ to a moving estimate of \texttt{q\_valid\_rate} so the static prior tracks Proposer health. The exploit-mode ablation (\S\ref{sec:empirical_ucb_negative}) shows that this reverse-imbalance pathway is a downstream \emph{symptom} of validity collapse rather than an independent driver: both empirical-UCB cells (with and without crossover) eventually enter the validity bottleneck and therefore eventually trip branch-weight inversion, while the ratchet cell preserves $q_{\mathrm{valid}}$ and never enters the inverted regime within our compute horizon. The branch-balance formula is therefore safe whenever the upstream Q-format manifold is preserved by the exploit term, and the rescue floor remains a defensive measure for runs that may dip into the empirical regime under different reward shapes.

\subsection{Static vs.\ Adaptive Branch Weighting}
\label{sec:branch_weight_ablation}

\S\ref{sec:joint_opt} introduces two branch-weighting designs after token-mass equalization. One configuration sets a fixed Proposer branch weight $w_Q=q_{\mathrm{scale}}=0.375$; the adaptive alternative equalizes Proposer and Solver advantage $\ell_2$-norms ($w_Q \propto r\cdot\mathrm{scale}_S\|U_S\|_2/\|U_Q\|_2$ with target ratio $r{=}0.5$, floor $w_Q^{\min}{=}0.125$). On paper the adaptive scheme is the more direct choice---it targets a constant relative gradient footprint between the two roles rather than a constant relative weight. Figure~\ref{fig:q_balance_comparison} compares the actor gradient norm of two Band-1-of-$K$ training runs configured with the two settings.

\begin{figure}[h]
    \centering
    \includegraphics[width=\linewidth]{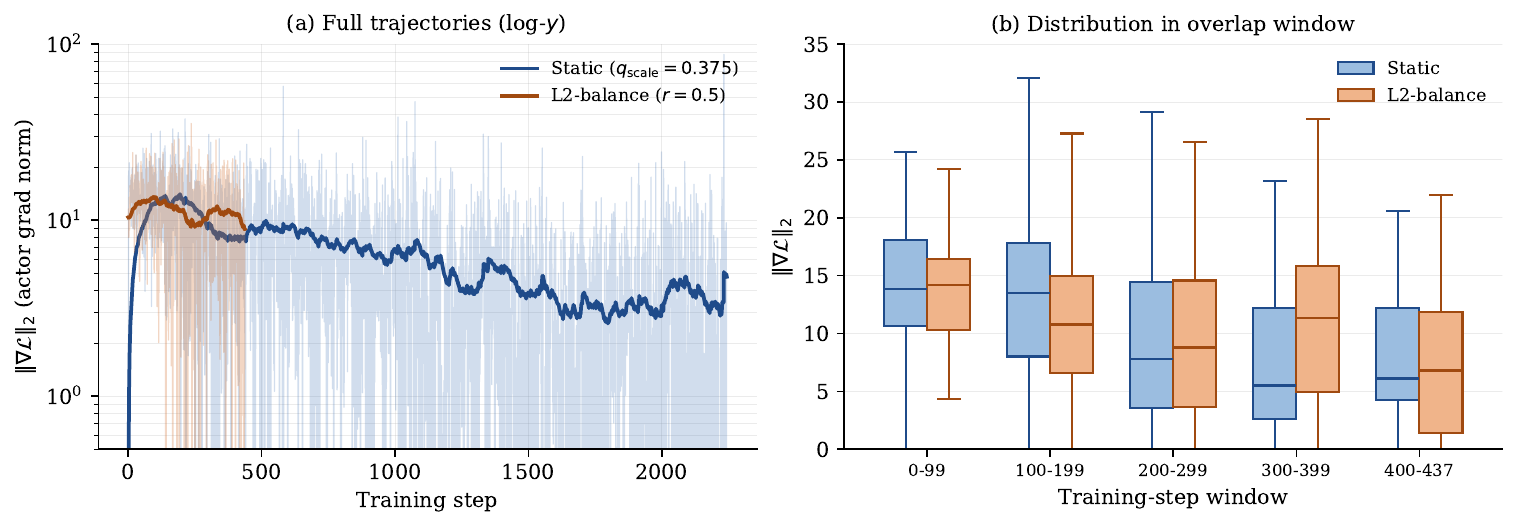}
    \caption{Actor gradient-norm trajectories of two Band-1-of-$K$ runs configured with the two Proposer/Solver branch-weighting settings. \textbf{(a)}~Full-horizon view (log-$y$): light traces are per-step values, solid lines are EMA ($\alpha{=}0.02$) for readability. The \emph{static} run ($q_{\mathrm{scale}}{=}0.375$) exhibits a clear long-horizon annealing pattern, with EMA grad-norm descending from $\sim\!13$ at step 100 to $\sim\!5$ by step 1500; the \emph{$\ell_2$-balance} run ($r{=}0.5$) instead plateaus in the $10$--$13$ band and was halted at step 437 without settling. \textbf{(b)}~Distribution of per-step grad-norm within the overlap window $[0,437]$, binned by 100-step windows. The two schemes have comparable early-phase spread, but the static scheme's median begins to decline around step 200 whereas the adaptive scheme's median stays elevated. Because the adaptive run did not reach a stable regime within our compute budget, we could not obtain a like-for-like comparison on downstream pass-rate metrics; we therefore treat this comparison as illustrative.}
    \label{fig:q_balance_comparison}
\end{figure}

\section{Prompt Templates}
\label{appPrompts}

\subsection{Proposer System Prompt (Abridged)}

The Proposer system prompt enforces a 7-label structured chain-of-thought discipline. The key elements are shown below (minor formatting details omitted for brevity):

\begin{lstlisting}[basicstyle=\ttfamily\scriptsize, breaklines=true, breakatwhitespace=true, columns=flexible]
You are an expert Verus (Rust) programmer specializing in formal verification.
Your job is to generate one new curriculum question, not a solution.

Output contract:
- Output exactly one <think>...</think> block followed by one ```rust``` block.
- The rust block must contain exactly one target item (fn/spec fn/proof fn).
- Do not include a function body.

Reasoning discipline inside <think>:
Use exactly these seven labels in this order:
1. Parent core:       Identify the core objective in plain language.
2. Direction target:  Restate the requested direction (up/down).
3. Solver-skill focus: Name the one reusable solver obligation the child should teach.
4. Minimal decomposition: For up, explain why no further decomposition is needed;
                          for down, isolate the smallest blocker subproblem.
5. Stepping-stone rationale: Explain why this child is the right next step.
6. Next-step child choice: Formulate the minimal plan by choosing the best child.
7. Check: Verify the full reasoning chain for consistency.

Hard bans:
- No rename-only mutation.
- No child whose extra property is already implied by the seed.
- No child whose difficulty cannot be tied to a reusable S-Gen proof skill.
\end{lstlisting}

\subsection{Proposer Curriculum Template (Abridged)}

The curriculum template provides a seed specification and a direction (up or down), determined adaptively by the system (\S\ref{sec:adaptive_curriculum}):

\begin{lstlisting}[basicstyle=\ttfamily\scriptsize, breaklines=true, breakatwhitespace=true, columns=flexible]
Below is a seed problem from the existing pool:

```rust
{seed_code}
```

Direction: {direction}

Task:
- Generate one new Verus target item that is a plausible next curriculum step.
- Make one focused semantic change rather than a rename or superficial mutation.
- If Direction: up, keep the parent family and add one minimal semantic extension.
- If Direction: down, keep the parent family and expose one smaller
  blocker-focused stepping stone.
- The new difficulty must correspond to one reusable solver skill.
\end{lstlisting}

\section{Qualitative Examples from Late-Stage Rollouts}
\label{appExamples}

To verify that longer late-stage responses correspond to meaningful reasoning rather than degenerate repetition, we manually inspected solved long responses from the rollout log at training step 1422. In this batch, the response-length distribution had median 1283 characters, 90th percentile 1750 characters, and maximum 3431 characters. The longest \emph{solved} responses were generally well-formed proof-oriented explanations paired with concise Verus implementations, rather than the pathological repetition patterns observed in some earlier unstable phases.

\subsection{Abridged Solver CoT Examples}

\paragraph{Example A: Minimum Second Value Subject to a Threshold Constraint}
The following is an abridged excerpt from a solved long response (step 1368, sample 190):
\begin{lstlisting}[basicstyle=\ttfamily\scriptsize, breaklines=true, breakatwhitespace=true, columns=flexible, frame=single]
The function returns a value `first_of_min_second` that satisfies the
postcondition. Specifically, it returns an index `i` such that
`0 <= i < s.len()` and `first_of_min_second == s[i as int][0]`,
and it satisfies the universal postcondition that for all indices `j`,
either `s[i as int][1] <= s[j as int][1]` and
`s[i as int][0] > threshold`, or `i == j`.
</think>

```rust
let mut min_index: usize = 0;
    for i in 1..s.len()
        invariant
            1 <= i <= s.len(),
            min_index < i,
            forall|j: nat| 0 <= j && j < i ==>
                s[min_index as int][1] <= s[j as int][1]
                && s[min_index as int][0] > threshold,
            1 <= i <= s.len(),
            forall|i: nat| i < s.len() ==> s[i as int].len() >= 2,
            forall|i: nat| i < s.len() ==>
                s[i as int][1] <= s[i as int][1]
                && s[i as int][0] > threshold,
        decreases s.len() - i
    {
        if s[i][1] < s[min_index][1]
            || (s[i][1] == s[min_index][1] && s[i][0] > s[min_index][0])
        {
            min_index = i;
        }
    }
    s[min_index][0]
}
```

This function initializes `min_index` to 0 and iterates over the rest
of the vector, updating `min_index` whenever it finds a pair of indices
`i` and `j` where `s[i][1] < s[j][1]` or
`s[i][1] == s[j][1] && s[i][0] > s[j][0]`. The postcondition is
maintained through modifications to `min_index` within the loop, and
the function returns `s[min_index][0]` after completion.
\end{lstlisting}

\paragraph{Example B: Minimum-Length Nonempty Sublist}
The following is another abridged excerpt from a solved long response (step 1413, sample 253):
\begin{lstlisting}[basicstyle=\ttfamily\scriptsize, breaklines=true, breakatwhitespace=true, columns=flexible, frame=single]
The function returns a vector `min_sublist` such that:
- `min_sublist.len()` is less than or equal to the length of some
  element in `s`.
- There exists an index `i` with `0 <= i < s.len()` where
  `min_sublist.len()` equals the length of `s[i]`.
- For all indices `i` with `0 <= i < s.len()`, `min_sublist.len()`
  is less than or equal to the length of `s[i]`.
- The function satisfies all preconditions.
</think>

```rust
let mut min_sublist: Vec<int> = s[0].clone();
    for i in 1..s.len()
        invariant
            1 <= i <= s.len(),
            0 <= i,
            i <= s.len(),
            min_sublist.len() <= s[0].len(),
            exists|k: int| 0 <= k && k < i as int
                && min_sublist.len() == s[k].len(),
            forall|j: int| 0 <= j && j < i ==>
                min_sublist.len() <= s[j].len(),
            s.len() > 0,
        decreases s.len() - i
    {
        if s[i].len() < min_sublist.len()
        {
            min_sublist = s[i].clone();
        }
    }
    min_sublist
}
```

This function initializes `min_sublist` with the first vector in `s`
and iterates over the rest, updating the invariant to maintain
correctness. The loop maintains the existence of a valid candidate
by updating `min_sublist` only when a shorter sublist is found.
\end{lstlisting}
\subsection{Interpretation}

These late-stage solved examples support two qualitative observations:
\begin{enumerate}
    \item \textbf{Longer responses are not necessarily pathological.} In late training, some increase in response length reflects explicit invariant narration and postcondition checking rather than raw repetition.
    \item \textbf{The dominant style is ``spec recap + invariant-driven implementation + brief proof sketch.''} This is consistent with the intended solver behavior in ANCORA, where reusable proof obligations are surfaced explicitly.
\end{enumerate}

At the same time, these examples also show that the model often produces \emph{over-explained} solutions. The extra length typically comes from rephrasing the same postcondition in prose before and after the code, rather than from discovering substantially new proof content.

\section{Curriculum Bandit Analysis: Discounted UCB on a Lipschitz DAG}
\label{app:pruned_zpd_uct_dag}

This appendix formalizes the curriculum explorer used by ANCORA as a \emph{non-stationary discounted-UCB problem on a Lipschitz arm space}, and derives a surrogate regret object that the implementation qualitatively targets. A flat epochized UCB1 framing over an active set $V_e$ would make the vanilla bound vacuous because the online system grows its arm set at every step ($K_e$ linear in $T$). The genuine mathematical object exercised by the production implementation is a \emph{discounted} UCB recursion on arms that live in an $\epsilon$-packing of a similarity metric space, with solved-only admission guaranteeing boundedness of rewards and with topological hard-keep restricted to being an optional long-horizon cleanup policy rather than a load-bearing theoretical primitive. The rest of this appendix makes each of these claims precise and gives the corresponding qualitative regret scaling.

\subsection{Objective and Teacher Utility}

The outer objective of the curriculum RL system is solver improvement,
\begin{equation}
J(\pi) \;=\; \tau_T - \tau_0, \qquad \tau_t \;=\; \mathbb{E}_{q \sim \mathcal{D}_{\mathrm{test}}}\!\bigl[\mathrm{Pass}(S_t, q)\bigr],
\label{eq:outer_objective}
\end{equation}
where $S_t$ denotes the solver at step $t$. Since a UCB analysis requires a bounded per-rollout reward signal, we introduce a surrogate \emph{teacher utility}
\begin{equation}
u_\tau(p) \;=\; 4\,p(1-p)\,\exp\!\left(-\frac{(p-\tau)^2}{2\sigma^2}\right) \;\in\; [0,1],
\label{eq:teacher_utility}
\end{equation}
where $p$ is the rollout pass-rate of the generated question and $\tau$ is the current solver-ability estimate. The factor $4p(1-p)$ favors ZPD-edge questions (maximum at $p=\tfrac12$); the Gaussian localizes around the frontier; and the normalization $u_\tau \in [0,1]$ is what the analysis consumes. The reward multiplexer \texttt{q\_reward\_mode=band\_1of8} replaces $4p(1-p)$ with its sparse quantile variant, and \texttt{q\_reward\_mode=exp\_decay} with a monotone frontier kernel; both remain in $[0,1]$, so the surrogate scaling below is unchanged under any of these choices.

\subsection{Solved-only Admission and Metric Arm Space}
\label{subsec:admission_metric}

Let $\mathcal{X}$ be the space of candidate Verus specifications and let $d:\mathcal{X}\times\mathcal{X}\to[0,1]$ be the $\mathrm{MinHash}_{128}$ Jaccard distance used by \texttt{SeedPoolManager.question\_similarity}. The admission rule of \texttt{SeedPoolManager.add\_seed} only inserts $q\in\mathcal{X}$ if the trainer stored a non-empty \texttt{best\_ans} string for $q$; \texttt{best\_ans} is only set when \texttt{strict\_pass=True} on at least one solver rollout of $q$.

\begin{lemma}[Solved-only admission invariant]
\label{lem:solved_only}
Let $P_t\subset\mathcal{X}$ be the seed pool at step $t$, and let $P_t^{\mathrm{ext}}\subset P_t$ be the subset of non-original nodes admitted after step $0$. Then every $q\in P_t^{\mathrm{ext}}$ has been solved at least once by $S_{t'}$ for some $t'\le t$. The per-arm reward variable $R(q)\in[0,1]$ is bounded for every $q\in P_t$ (both extended and original), but strict-positivity of the current-model pass-rate statistic is guaranteed only on $P_t^{\mathrm{ext}}$: original root anchors are loaded with \texttt{pass\_rate}$=0$ under the fresh-start reset in \texttt{SeedPoolManager.\_load\_from\_file} (invoked with \texttt{reset\_pass\_rate=True} on cold start) and acquire positive statistics only after visitation by the current solver.
\end{lemma}

The lemma follows from the trainer-side admission guard, which only inserts new non-original nodes for which \texttt{best\_ans} is non-empty; \texttt{best\_ans} is in turn only set inside the \texttt{strict\_pass=True} branch of the solver-verification loop. Its consequence for the regret analysis is that admission itself---rather than arm retirement---ensures every non-original arm carries bounded reward $R(q)\in[0,1]$ with a strictly positive historical pass, so no separate retirement lemma is required.

\begin{lemma}[$\epsilon$-packing of the pool]
\label{lem:epsilon_packing}
Let $\epsilon=1-\texttt{max\_pool\_sim}$. For every $t$ and every $q,q'\in P_t$ with $q\ne q'$, $d(q,q')>\epsilon$. Consequently $|P_t|\le \mathcal{N}(\mathcal{X},d,\epsilon/2)$, the $\epsilon/2$-covering number of $(\mathcal{X},d)$.
\end{lemma}

\begin{proof}
The admission gate \texttt{SeedPoolGate.check(\ldots, max\_pool\_sim)} rejects any candidate whose maximum Jaccard similarity to the current pool exceeds $\texttt{max\_pool\_sim}$, i.e., whose Jaccard distance to at least one existing member is $\le\epsilon$. Hence accepted pool elements are pairwise separated by more than $\epsilon$, so $P_t$ is an $\epsilon$-packing. The standard packing-covering inequality bounds the maximum size of an $\epsilon$-packing by the $\epsilon/2$-covering number.
\end{proof}

Under our configuration, $\texttt{max\_pool\_sim}=0.72$, so $\epsilon=0.28$; empirically $|P_{200}|\approx 690$, which is compatible with the covering-number bound as long as $\mathcal{N}(\mathcal{X},d,0.14)$ is of this order. The Lipschitz arm-space bandit literature~\cite{kleinberg2008multiarmed} shows that all subsequent regret analysis depends on $\mathcal{X}$ only through the relevant covering number.

\subsection{Discounted UCB Recursion}
\label{subsec:discounted_ucb_model}

For analysis we collapse the production DAG traversal into a node-level discounted-UCB surrogate that treats each $q\in P_t$ as a single arm pulled with a global UCB index. The production \texttt{SeedPoolManager.ucb\_sample} routine refines this surrogate into hierarchical root/stay/child actions with parent-local visit denominators and several bounded additive priors; these deviations are catalogued in \S\ref{subsec:production_deviations} and do not invalidate the qualitative regret object derived below.

At step $t$, the explorer maintains for each arm $q\in P_t$ a discounted visit count $n_t(q)$ and discounted reward sum $s_t(q)$ with decay $\gamma_{\mathrm{UCB}}\in(0,1)$ (Hydra: \texttt{ucb\_stat\_decay}, e.g.\ $0.995$):
\begin{equation}
n_{t+1}(q) = \gamma_{\mathrm{UCB}}\,n_t(q) + \mathbbm{1}[A_t=q], \qquad
s_{t+1}(q) = \gamma_{\mathrm{UCB}}\,s_t(q) + R_t\,\mathbbm{1}[A_t=q].
\label{eq:discount_recursion}
\end{equation}
The discounted empirical mean is $\hat\mu_t(q)=s_t(q)/n_t(q)$, and the discounted UCB index is
\begin{equation}
U_t(q) \;=\; \hat\mu_t(q) \;+\; c\sqrt{\frac{\log n_t^{\mathrm{tot}}}{n_t(q)}}, \qquad n_t^{\mathrm{tot}}=\sum_{q\in P_t}n_t(q),
\label{eq:discount_ucb_index}
\end{equation}
with exploration coefficient $c=\texttt{ucb\_c}$. The effective window of the discount is $\tau_\gamma\equiv 1/(1-\gamma_{\mathrm{UCB}})=200$ steps---each arm's statistics geometrically forget observations older than $\tau_\gamma$. The implementation in \texttt{SeedPoolManager.decay\_pass\_rates} multiplies \texttt{visit\_counts}, \texttt{value\_sums}, \texttt{stay\_visit\_counts}, \texttt{edge\_visit\_counts} and \texttt{\_total\_visits} by $\gamma_{\mathrm{UCB}}$ at the end of every step, which realizes~\eqref{eq:discount_recursion} exactly.

Crucially, we use the discount rate \texttt{ucb\_stat\_decay} \emph{only} as a bandit non-stationarity knob: the trainer tracks a separate forgetting rate \texttt{pass\_rate\_decay} (default $0.996$) for its operational model of solver ability, so that the two decays can be tuned independently.

\subsection{Two-Gate Bounded Reward}

Given a candidate $x$ with rollout pass-rate $p(x)$, the trainer assigns the two-gate reward
\begin{equation}
R_t(x) \;=\; H(x)\,\bigl((1-\beta)\,u_{\tau_t}(p(x)) + \beta\,A(x)\bigr), \qquad 0 \le \beta \le 0.1,
\label{eq:two_gate_reward}
\end{equation}
where $H(x)\in\{0,1\}$ is the reusable hard gate (struct-sim $<1$, non-duplicate) and $A(x)\in\{0,1\}$ is the pool-admission gate. Boundedness $R_t\in[0,1]$ is the only property the surrogate argument exploits. The $\beta$-mixing keeps credit flowing to productive intermediate nodes even when rollouts score $u_\tau=0$ but still successfully admit a novel seed. We emphasize that $R_t(x)$ here is the \emph{backed-up reward} credited along the sampled path and used to update $s_t(q)$ in Equation~\eqref{eq:discount_recursion}; it is distinct from the selection-time UCB index $U_t(q)$ in Equation~\eqref{eq:discount_ucb_index}, and from the bounded additive selection priors (ZPD kernel, kinship prior, neighborhood cold-start fallback) that the production code adds to $U_t$ but not to $R_t$ (\S\ref{subsec:production_deviations}). The surrogate argument only requires $R_t\in[0,1]$; it places no restriction on the priors attached to $U_t$ beyond boundedness.

After a rollout of depth $L$ along a sampled action path $(v_0,a_0),\ldots,(v_L,a_L)$, discounted path backup credits each ancestor action with
\begin{equation}
Y_i \;=\; \gamma^{L-i}\,R_t(x), \qquad 0 \le \gamma \le 1.
\label{eq:discounted_backup}
\end{equation}
Because \eqref{eq:discounted_backup} is itself in $[0,1]$ and is a deterministic function of the rollout outcome, it merely rebrands the terminal reward at each intermediate node and does not break the bandit structure.

\subsection{Piecewise-Stationary Curriculum and Variation Budget}
\label{subsec:piecewise_stationary}

Because the solver $S_t$ changes from step to step, the expected reward $\mu_t(q)=\mathbb{E}[R_t\mid A_t=q]$ is non-stationary. Following Garivier and Moulines~\cite{garivier2011discounted}, we model this non-stationarity by a \emph{variation budget}
\begin{equation}
\Upsilon_T \;=\; \sum_{t=1}^{T-1}\,\max_{q\in P_t\cap P_{t+1}}\,\bigl|\mu_{t+1}(q)-\mu_t(q)\bigr|,
\label{eq:variation_budget}
\end{equation}
which counts the total pointwise drift of the reward landscape over $T$ steps. The curriculum assumption is that $\mu_t$ is piecewise stationary with $\Upsilon_T$-controlled break magnitude, an assumption satisfied whenever solver improvements are Lipschitz-in-$t$---e.g., as long as per-step policy updates are bounded. In particular, we do \emph{not} assume an exogenous epoch partition; the analysis is phrased directly against~\eqref{eq:variation_budget}.

\subsection{Recent-Arm Proxy and Covering Number}
\label{subsec:active_arm_bound}

Because statistics decay geometrically with rate $\gamma_{\mathrm{UCB}}$, an arm last pulled $s$ steps ago retains a discount weight $\gamma_{\mathrm{UCB}}^s$ in~\eqref{eq:discount_ucb_index}. This does \emph{not} literally make old arms inactive: as $n_t(q)$ decays, the UCB bonus can make a neglected arm attractive again. We therefore use the following recent-arm count only as an analytical proxy for the finite set that dominates a bounded compute window. Fix a tolerance $\delta\in(0,1)$ and define the $\delta$ horizon
\[
\tau_\delta \;=\; \frac{\log(1/\delta)}{1-\gamma_{\mathrm{UCB}}}\quad\Longleftrightarrow\quad \gamma_{\mathrm{UCB}}^{\tau_\delta}\le \delta,
\]
so that an observation older than $\tau_\delta$ retains discount weight at most $\delta$. The natural-time scale $\tau_\gamma=1/(1{-}\gamma_{\mathrm{UCB}})$ corresponds to $\delta=e^{-1}\approx 0.37$, which is \emph{not} small; we therefore quote scalings parametrized by $\delta$ and instantiate them with a moderately small choice, e.g.\ $\delta=e^{-3}\approx 0.05$ giving $\tau_\delta\approx 3\tau_\gamma$. Let
\[
K_\gamma(\delta) \;=\; \sup_{t\le T}\bigl|\{q\in P_t:\, q\ \text{has been pulled within the last } \tau_\delta\ \text{steps}\}\bigr|.
\]
If the per-step branching factor of seed insertion is bounded by $B$ (Hydra: \texttt{q\_gen\_n}, value $8$), each step can activate at most $B$ fresh arms; combined with Lemma~\ref{lem:epsilon_packing} this gives:

\begin{corollary}[Recent-arm count, $\delta$-tolerance]
\label{cor:arm_count_bound}
$K_\gamma(\delta) \le \min\bigl(B\,\tau_\delta,\ \mathcal{N}(\mathcal{X},d,\epsilon/2)\bigr)$. With $B=8$, $\tau_\gamma=200$, $\epsilon=0.28$ and the choice $\delta=e^{-3}$ (so $\tau_\delta\approx 600$), this gives $K_\gamma(\delta)\le 4800$ independent of the horizon $T$. The bound is qualitative: $\delta$ trades off the looseness of the effective-set approximation against the size of $\tau_\delta$, and the bound below should be read with this $\delta$-dependence in mind.
\end{corollary}

\subsection{Main Surrogate Scaling: Discounted UCB}

\begin{proposition}[Surrogate discounted-UCB scaling on a Lipschitz DAG]
\label{thm:discounted_ucb_main}
Fix $\delta\in(0,1)$. Under Lemmas~\ref{lem:solved_only}--\ref{lem:epsilon_packing}, the variation budget~\eqref{eq:variation_budget}, and the recent-arm proxy of Corollary~\ref{cor:arm_count_bound}, the discounted-UCB sampler~\eqref{eq:discount_recursion}--\eqref{eq:discount_ucb_index} with discount $\gamma_{\mathrm{UCB}}=1-1/\tau_\gamma$ motivates the following gap-independent surrogate scaling, for some universal $C_1,C_2>0$:
\begin{equation}
\mathbb{E}[\mathrm{Reg}(T)] \;\lesssim\; C_1\,\sqrt{\frac{K_\gamma(\delta)\,T\,\log T}{1-\gamma_{\mathrm{UCB}}}} \;+\; C_2\,\frac{\Upsilon_T}{1-\gamma_{\mathrm{UCB}}} \;+\; \mathcal{O}(\delta\,T).
\label{eq:discounted_ucb_bound}
\end{equation}
The $\mathcal{O}(\delta T)$ slack records the cost of using a recent-arm proxy rather than the literal decayed-UCB active set. We state this as a surrogate scaling rather than a literal theorem: the underlying Garivier--Moulines (2011) result is for a fixed-cardinality finite arm set, while our active arm set is a growing metric $\epsilon$-packing and neglected arms are not truly removed from the UCB index.
\end{proposition}

\begin{proof}[Heuristic derivation]
This adapts the gap-independent discounted-UCB scaling of Garivier--Moulines (2011) to our setting. Their finite-arm result places the arm count inside the square root for a fixed cardinality $K$. We replace $K$ by the recent-arm proxy $K_\gamma(\delta)$ of Corollary~\ref{cor:arm_count_bound}. This is an approximation: arms not pulled within $\tau_\delta$ steps have old observations downweighted by at most $\delta$, but their UCB bonus can still make them selectable, so the proxy should be read as a finite-budget diagnostic rather than a literal runtime truncation. The arm space being a metric $\epsilon$-packing rather than a flat discrete set is handled by the zooming argument of Kleinberg--Slivkins--Upfal (2008, Theorem~3.1): the per-step regret against the continuous optimum is dominated by the covering number $\mathcal{N}(\mathcal{X},d,\epsilon/2)$, which is absorbed into $K_\gamma(\delta)$ by Lemma~\ref{lem:epsilon_packing} and Corollary~\ref{cor:arm_count_bound}. Boundedness $R_t\in[0,1]$ needed for Hoeffding concentration is furnished by Lemma~\ref{lem:solved_only} and the gating construction~\eqref{eq:two_gate_reward}. A fully rigorous transfer would additionally require quantifying the bias introduced by the time-varying active arm set; we leave this to future work and read~\eqref{eq:discounted_ucb_bound} as a qualitative scaling rather than a sharp inequality.
\end{proof}

\paragraph{Interpretation}
Plugging in $\tau_\gamma=200$ at $\delta=e^{-3}$ (so $\tau_\delta\approx 600$, $K_\gamma(\delta)\le 4800$) and $T=2000$ yields a dominant term $C_1\sqrt{4800\cdot 2000\cdot\log(2000)/0.005}\approx 1.2\!\times\!10^5\cdot C_1$, with truncation slack $\mathcal{O}(\delta T)\approx 100$. We stress that this numerical evaluation is \emph{not} a useful per-step regret guarantee at our compute horizon. The load-bearing content of the bound is therefore qualitative: (i)~it replaces lifetime arm count $|P_T|$ with a recent-arm/covering proxy, (ii)~its $T$-dependence is $\tilde O(\sqrt{T})$ rather than the linear-in-$T$ regret of unanchored exploration, and (iii)~it tolerates the empirical growth $|P_t|$: $234\to 690$ we observe across $200$ steps, as long as the admission gate continues to enforce the packing invariant.

\subsection{Arm Merging and Depth Cap as Lipschitz Operations}
\label{subsec:arm_merging}

Two operations further compress the arm space without invalidating Proposition~\ref{thm:discounted_ucb_main}.

\textbf{Absorb (\texttt{mcts\_absorb\_sim\_threshold}).} When a child $q'$ of $q$ has $d(q,q')\le 1-\rho$ with $\rho=\texttt{mcts\_absorb\_sim\_threshold}=0.60$, the child statistics are merged into the parent. Since $d$ is a metric (MinHash Jaccard over 128 permutations), triangle inequality gives $|\mu_t(q)-\mu_t(q')|\le L\cdot d(q,q')$ for any $L$-Lipschitz reward model; absorb therefore introduces a bounded bias $\le L(1-\rho)$ per merge event, which is absorbed into $C_1$.

\textbf{Depth cap (\texttt{mcts\_depth\_cap}).} Hard-capping rollout depth at $D=6$ bounds the path-backup exponent in~\eqref{eq:discounted_backup} and keeps the per-step expected reward in $[0,1]$. The cap does not affect the UCB analysis---it only affects the \emph{mapping} from rollouts to arm pulls.

\textbf{Symmetric crossover (\texttt{q\_crossover\_ratio=0.5}).} Secondary crossover credit adds a second, symmetrized backup path per rollout. Because both paths still land in $[0,1]$, the analysis is unchanged; the constant $C_1$ grows by a factor of at most two.

\subsection{Optional Long-horizon Cleanup: \texttt{prune\_unsolvable}}
\label{subsec:prune_demoted}

Note that \texttt{prune\_unsolvable} is \emph{not} load-bearing in the main theorem. It performs an optional garbage collection of nodes with cumulative discounted visits $\ge\texttt{mcts\_prune\_min\_visits}=20$ that are also non-original and lie outside $\mathrm{HardKeep}$. Since admission is already solved-only (Lemma~\ref{lem:solved_only}), the pool never contains the no-signal arms that pruning was originally designed to remove; empirically \texttt{prune\_unsolvable} fires zero times in our main run across $200$ steps. We keep it available as a long-horizon sanity cleanup but no longer attribute regret consequences to it.

\subsection{Lower Bound}

\begin{theorem}[Minimax lower bound]
\label{thm:pzpd_lower}
Any algorithm for the problem class of Sections~\ref{subsec:admission_metric}--\ref{subsec:piecewise_stationary} has worst-case regret $\Omega\!\bigl(\sqrt{K_\gamma T}\bigr)$.
\end{theorem}

\begin{proof}
Our problem class contains the special case $\Upsilon_T=0$, $\gamma_{\mathrm{UCB}}\to 1$, with $K_\gamma$ arms that never absorb or merge. This is the finite stochastic $K_\gamma$-armed bandit with rewards in $[0,1]$, for which the standard $\Omega(\sqrt{K_\gamma T})$ lower bound holds~\cite{auer2002finitetime}.
\end{proof}

\subsection{Assumptions Actually Used}
\label{subsec:pzpd_gaps}

Proposition~\ref{thm:discounted_ucb_main} relies on five assumptions, each verified against the production code path:

\begin{enumerate}[leftmargin=1.4em]
\item \textbf{Boundedness} ($R_t\in[0,1]$): enforced by the two-gate construction~\eqref{eq:two_gate_reward}; the three \texttt{q\_reward\_mode} variants (\texttt{band\_1of8}, \texttt{4p1p}, \texttt{exp\_decay}) each return values in $[0,1]$.
\item \textbf{Solved-only admission} (Lemma~\ref{lem:solved_only}): enforced by the \texttt{best\_ans} guard in the trainer admission path, combined with the \texttt{strict\_pass} requirement of the solver-verification loop.
\item \textbf{$\epsilon$-packing} (Lemma~\ref{lem:epsilon_packing}): enforced by \texttt{SeedPoolGate.check} with dedup thresholds (e.g., \texttt{max\_pool\_sim=0.72}, \texttt{max\_sibling\_sim=0.55}, \texttt{duplicate\_struct\_sim\_threshold=0.90}).
\item \textbf{Piecewise stationarity with variation budget $\Upsilon_T$}: not verifiable from code alone; this is an assumption on the \emph{solver learning dynamics}, equivalent to requiring that policy updates do not induce unbounded drift between successive steps. Empirically $\hat\tau_t=\texttt{s\_ability\_ema}$ evolves slowly (per-step changes $<0.05$ in the first 200 steps), giving an order-of-magnitude check $\Upsilon_T\lesssim 10$ in our main run. The ablation in \S\ref{sec:empirical_ucb_negative} shows that this assumption is \emph{not} automatic: in that configuration, runs with the discounted empirical mean $\hat\mu_t(q)$ showed Proposer drift within 300--420 steps, whereas the run with the max-ratcheted exploit term did not.
\item \textbf{Recent-arm proxy}: $K_\gamma(\delta)$ is an analytical proxy, not a runtime truncation. Older arms have geometrically decaying statistics in~\eqref{eq:discount_ucb_index} and may become attractive again through the exploration bonus; the $\mathcal{O}(\delta T)$ term in~\eqref{eq:discounted_ucb_bound} records this approximation rather than claiming those arms vanish.
\end{enumerate}

The five assumptions above suffice; in particular, no clean epoch partition with UCB statistic resets, no UCB-LCB certificate for safe retirement, and no central load-bearing role for topological hard-keep is required for Proposition~\ref{thm:discounted_ucb_main}.

\subsection{Production Deviations from the Surrogate}
\label{subsec:production_deviations}

Proposition~\ref{thm:discounted_ucb_main} analyzes a clean node-level discounted-UCB surrogate: a flat arm set $P_t$, a single UCB index $U_t(q)$ with a global visit denominator, and the discounted empirical mean $\hat\mu_t(q)=s_t(q)/n_t(q)$ as the exploit term. The production selector (\texttt{SeedPoolManager.ucb\_sample}) is a constrained DAG-UCT controller that inherits this surrogate's regret intuition but deviates from it in six bounded ways; we list them here so the analysis-to-code correspondence is honest rather than overstated.

\begin{enumerate}[leftmargin=1.4em]
\item \textbf{Hierarchical root/stay/child actions.} \texttt{mcts\_select} first chooses a root via \texttt{\_root\_score} (global denominator $\log(\texttt{\_total\_visits}{+}1)/(n_t(q){+}1)$), then at each descent step compares \texttt{\_stay\_score}(q) against $\max_{q'}$\texttt{\_child\_score}(q,q'). Stay and child scores use \emph{parent-local} denominators $\log(n_t(q_{\mathrm{parent}}){+}1)/(n_t^{\mathrm{stay}}(q){+}1)$ and $\log(n_t(q_{\mathrm{parent}}){+}1)/(n_t^{\mathrm{edge}}(q_{\mathrm{parent}},q'){+}1)$ respectively, rather than the single $n_t^{\mathrm{tot}}$ of Equation~\eqref{eq:discount_ucb_index}.
\item \textbf{Exploit mode.} \texttt{ucb\_exploit\_mode} selects a max-ratcheted or an empirical exploit term; in the configuration of \S\ref{sec:empirical_ucb_negative}, the empirical variants collapsed within 300--420 steps, while the ratcheted variant did not.
\item \textbf{Cold-start neighborhood prior.} \texttt{\_node\_value} returns $s_t(q)/n_t(q)$ only on visited arms ($n_t(q){>}0$); on unvisited arms it falls back through a three-step \texttt{neighborhood\_pass\_rate} chain (own \texttt{pass\_rate} $\to$ mean child \texttt{pass\_rate} $\to$ global positive-pass-rate mean $\to$ constant $0.15$). The surrogate abstracts this warm-start prior away; in the code it is what allows freshly inserted nodes to be selectable before they accumulate backed-up reward.
\item \textbf{Additive selection priors.} Under \texttt{empirical} mode, \texttt{\_exploit} returns $\hat\mu_t(q) + \lambda_{\mathrm{ZPD}}\,\mathrm{ZPD}(q,\hat{a}_t)$; and \texttt{\_root\_score}/\texttt{\_child\_score}/\texttt{\_stay\_score} each further add a bounded kinship prior $\lambda_{\mathrm{kin}}\,\bar\mu(\mathrm{KinN}(q))$. These are \emph{selection-time} bonuses attached to $U_t$, distinct from the backed-up reward $R_t$ of Equation~\eqref{eq:two_gate_reward}. All three priors lie in $[0,1]$ up to their (bounded) weights.
\item \textbf{Adaptive difficulty window and root quota.} Each step, an $\alpha{=}\texttt{ucb\_root\_ratio}$ fraction of prompts is reserved for root-level sampling (a configurable fraction, with the remainder used for DAG descent), and candidates are restricted to the difficulty window $[\hat{a}_t{-}\delta_{\mathrm{ZPD}}/2,\,\hat{a}_t{+}\delta_{\mathrm{ZPD}}/2]$ around the solver-ability estimate $\hat{a}_t$. Both constrain the feasible action set; neither alters the arm space $\mathcal{X}$ or the reward gate, so boundedness and $\epsilon$-packing remain intact.
\item \textbf{Batch repulsion and descend margin.} Within-batch penalties discourage repeated root/edge/stay actions, and \texttt{ucb\_descend\_margin} biases the stay-vs-child comparison by a small additive constant. These alter selection probabilities inside a batch but preserve the one-step reward boundedness the surrogate argument uses.
\end{enumerate}

All six deviations are bounded admissible perturbations of the surrogate: the hierarchical traversal reduces to root-level UCB at depth zero, the additive priors are bounded, and the adaptive window and batch repulsion are feasibility constraints that leave the arm space and the backed-up-reward gate unchanged. We do not claim these deviations preserve the exact constants in Equation~\eqref{eq:discounted_ucb_bound}; we claim only that none of them invalidates the qualitative regret object (covering-number-bounded, variation-sensitive, and sub-linear under bounded $\Upsilon_T$).

\subsection{Closest Implementation Counterparts}

For each theory object we give the closest implementation site in the released trainer snapshot (\texttt{code/trainer/} in the code repository). Several correspondences are approximate rather than literal: the surrogate UCB index uses parent-local denominators and additive priors in the code (\S\ref{subsec:production_deviations}); the teacher utility $u_\tau$ is realized as a reward mode used in backed-up tree reward, separately from the ZPD kernel used as a selection prior; and solved-only admission is enforced by the trainer-side admission path, not by a single function call.
\begin{center}
\small
\resizebox{\textwidth}{!}{%
\begin{tabular}{lll}
\toprule
Theory object & Closest implementation site & Example value \\
\midrule
$u_\tau(p)$ teacher utility (backed up) & \texttt{q\_reward\_mode}, used in \texttt{tree\_reward} backup & \texttt{band\_1of8} \\
ZPD selection prior & \texttt{SeedPoolManager.\_zpd\_exploit} & $\sigma=0.15$, $\lambda_{\mathrm{ZPD}}$ via \texttt{zpd\_prior\_weight} \\
Discount recursion~\eqref{eq:discount_recursion} & \texttt{SeedPoolManager.decay\_pass\_rates} & $\gamma_{\mathrm{UCB}}=0.995$ \\
UCB index~\eqref{eq:discount_ucb_index} (surrogate) & \texttt{\_root\_score}/\texttt{\_child\_score}/\texttt{\_stay\_score} & $c=2.5$ \\
Root quota / DAG-descent split & \texttt{ucb\_root\_ratio} in \texttt{ucb\_sample} & configurable \\
MinHash Jaccard metric $d$ & \texttt{SeedPoolManager.question\_similarity} & 128 perms \\
Solved-only admission (Lemma~\ref{lem:solved_only}) & trainer-side admission path & \texttt{strict\_pass} only \\
$\epsilon$-packing gate (Lemma~\ref{lem:epsilon_packing}) & \texttt{SeedPoolGate.check} & $\epsilon=0.28$ \\
Effective horizon $\tau_\gamma$ & $1/(1-\gamma_{\mathrm{UCB}})$ & $\tau_\gamma=200$ \\
Branching bound $B$ & \texttt{q\_gen\_n} & $B=8$ \\
Two-gate reward~\eqref{eq:two_gate_reward} & \texttt{pool\_admit\_bonus} controls $\beta$ & $\beta=0.1$ \\
Arm-merging (absorb) & \texttt{mcts\_absorb\_sim\_threshold} & $\rho=0.60$ \\
Depth cap & \texttt{mcts\_depth\_cap} & $D=6$ \\
Symmetric crossover & \texttt{q\_crossover\_ratio} & $0.5$ \\
Optional cleanup (\S\ref{subsec:prune_demoted}) & \texttt{prune\_unsolvable(mode="zpd")} & never fires in our main run \\
\bottomrule
\end{tabular}
}
\end{center}

Plugging these values into Corollary~\ref{cor:arm_count_bound} gives a finite $K_\gamma(\delta)$ independent of $T$, and Proposition~\ref{thm:discounted_ucb_main} yields a sublinear-in-$T$ surrogate scaling for all horizons up to which the $\epsilon$-packing invariant holds. The scaling rests on the assumptions actually used---boundedness via solved-only admission (Lemma~\ref{lem:solved_only}), the $\epsilon$-packing arm space (Lemma~\ref{lem:epsilon_packing}), the piecewise-stationary variation budget $\Upsilon_T$~\eqref{eq:variation_budget}, and the recent-arm proxy that approximates arms outside the last $\tau_\delta$ steps---and we do not claim it follows from the lemmas alone.

\section{Solver-Only Reinforcement Learning Beyond Verus}
\label{app:beyond_verus}

With the same Qwen2.5-Coder-3B backbone, a 300-step online Solver-RL run with standard multi-turn Python tool calls and deterministic answer checking, trained on a filtered DAPO-Math-17k subset with exact evaluation-set overlaps removed, increased pass@1 on a fixed 256-problem MATH-500 evaluation set from 22.56\% at step~0 to 56.84\% at step~299, and on AMC23 from 11.25\% to 44.38\%. The before/after evaluations used the same problems, prompts, tool interface, decoding, and evaluator.

Separately, we trained one policy by alternating executable math problems and Verus tasks. The first run revealed a loophole: the model could pass a textual check by commenting out part of the specification. With a stricter check that restores the required specification before compilation, the rerun again improved on Verus while retaining or improving math performance, and tool-call failures decreased during training. Because the mixed-domain run used the available problem pool rather than a held-out split, we present it as training evidence rather than a generalization estimate. Neither experiment involves the Proposer or the curriculum, so together they show only that the verifier-driven Solver update is not specific to Verus.
\end{document}